\documentclass{article}
\usepackage{ijcai23}

\usepackage{times}
\usepackage{soul}
\usepackage{url}
\usepackage[hidelinks]{hyperref}
\usepackage[utf8]{inputenc}
\usepackage[small]{caption}
\usepackage{graphicx}
\usepackage{amsmath}
\usepackage{stmaryrd}
\usepackage{amsthm}
\usepackage{amssymb}
\usepackage{bbm}
\usepackage{xspace}
\usepackage{thm-restate}
\usepackage{booktabs}
\usepackage{siunitx}
\usepackage{algorithm}
\usepackage{algorithmic}
\usepackage{microtype}
\usepackage{xcolor}
\usepackage{enumitem}
\usepackage{pgf}
\usepackage[switch]{lineno}

\usepackage{misc2}

\newtheorem{example}{Example}
\newtheorem{theorem}{Theorem}
\newtheorem{definition}{Definition}
\newtheorem{lemma}{Lemma}
\newtheorem{prop}{Proposition}

\title{SAT-Based PAC Learning of Description Logic Concepts}

\author{
 Balder ten Cate$^1$
 \and
 Maurice Funk$^{2, 3}$\and
 Jean Christoph Jung$^{4}$\And
 Carsten Lutz$^{2, 3}$
 \affiliations
 $^1$ILLC, University of Amsterdam\\
 $^2$Leipzig University\\
 $^3$Center for Scalable Data Analytics and Artificial Intelligence (ScaDS.AI)\\
 $^4$TU Dortmund University
 \emails
 b.d.tencate@uva.nl,
 maurice.funk@uni-leipzig.de,
 jean.jung@tu-dortmund.de,
carsten.lutz@uni-leipzig.de
}

\newcommand{\systemname}{SPELL\xspace}

\begin{document}

\maketitle

\begin{abstract}
%
  We propose \emph{bounded fitting} as a scheme for learning
  description logic concepts in the presence of ontologies. A main
  advantage is that the resulting learning algorithms come with
  theoretical guarantees regarding their generalization to unseen
  examples in the sense of PAC learning. We prove that, in contrast,
  several other natural learning algorithms fail to provide such
  guarantees. As a further contribution, we present the system SPELL
  which efficiently implements bounded fitting for the description
  logic \ELHdr based on a SAT solver, and compare its performance to a
  state-of-the-art learner.
\end{abstract}

\section{Introduction}

In knowledge representation, the manual curation of knowledge bases
(KBs) is time consuming and expensive, making learning-based
approaches to knowledge acquisition an attractive alternative. We are
interested in description logics (DLs) where \emph{concepts} are an
important class of expressions, used for querying KBs and also as
central building blocks for ontologies. The subject of learning DL
concepts from labeled data examples has received great interest,
resulting in various implemented systems such as DL-Learner, DL-Foil,
and
YINYANG~\cite{DBLP:journals/ws/BuhmannLW16,DBLP:conf/ekaw/Fanizzi0dE18,DBLP:journals/apin/IannonePF07}.
These systems 
take a set of positively and negatively labeled examples and an
ontology \Omc, and try to construct a concept that fits the examples
w.r.t.\ \Omc. The related \emph{fitting problem}, which asks to decide
the existence of a fitting concept, has also been studied
intensely~\cite{DBLP:journals/ml/LehmannH10,DBLP:conf/ijcai/FunkJLPW19,DBLP:conf/kr/JungLPW21}.

The purpose of this paper is to propose a new approach to concept
learning in DLs that we call \emph{bounded fitting}, inspired by both
bounded model checking as known from systems verification
\cite{DBLP:conf/tacas/BiereCCZ99} and by Occam algorithms from
computational learning theory
\cite{DBLP:journals/jacm/BlumerEHW89}. 
The idea of bounded fitting is to search for a
fitting concept of bounded size, iteratively increasing the size bound
until a fitting is found. This approach has two main advantages, which
we discuss in the following.

First, it comes with formal guarantees regarding the generalization of
the returned concept from the training data to previously unseen
data. This is formalized by Valiant's framework of \emph{probably
  approximately correct (PAC)
  learning}~\cite{DBLP:journals/cacm/Valiant84}. Given sufficiently
many data examples sampled from an unknown distribution, bounded
fitting returns a concept that with high probability $\delta$ has a
classification error bounded by some small
$\epsilon$. 
It is well-known that PAC learning is intimately linked to Occam
algorithms which guarantee to find a hypothesis of small
size~\cite{DBLP:journals/jacm/BlumerEHW89,DBLP:journals/tcs/BoardP92}.
By design, algorithms following the bounded fitting paradigm are
Occam, and as a consequence the number of examples needed for
generalization depends only linearly on $1/\delta$, $1/\epsilon$, and
the size of the target concept to be learned. This generalization
guarantee holds independently of the DL used to formulate concepts and
ontologies.  In contrast, no formal generalization guarantees have
been established for DL concept learning approaches.


The second advantage is that, in important cases, bounded fitting
enables learning based on SAT solvers and thus leverages the practical
efficiency of these systems. We consider ontologies formulated in the
description logic \ELHdr and concepts formulated in \EL, which may be
viewed as a core of the ontology language OWL 2 EL. In this case, the
\emph{size-restricted fitting problem}, which is defined like the fitting
problem except that the maximum size of fitting concepts to be
considered is given as an additional input (in unary), is
\NPclass-complete; it is thus natural to implement bounded fitting using a SAT
solver. For comparison, we mention that the unbounded
fitting problem is \ExpTime-complete in this case~\cite{DBLP:conf/ijcai/FunkJLPW19}.

As a further contribution of the paper, we analyze the generalization
ability of other relevant approaches to constructing fitting
\EL-concepts. We start with algorithms that return fittings that are
`prominent' from a logical perspective in that they are most specific
or most general or of minimum quantifier depth among all fittings.
Algorithms with such characteristics and their applications are
discussed
in~\cite{pods2023}. Notably,
constructing fittings via direct products of positive examples yields
most specific fittings~\cite{DBLP:conf/ijcai/ZarriessT13,DBLP:conf/aaai/JungLW20}. Our result is that, even without ontologies,
these types of algorithms
are not \emph{sample-efficient}, that is, no polynomial amount of
positive and negative examples is sufficient to achieve
generalization in the PAC sense.

We next turn to algorithms based on so-called downward refinement
operators which underlie all implemented DL learning systems that we
are aware of. We consider two natural such operators that are rather
similar to one another and combine them
with a
breadth-first search strategy. The first operator can be described as
exploring `most-general specializations' of the current hypotheses and
the second one does the same, but is made `artificially Occam' (with,
most likely, a negative impact on practicality). We prove that while
the first operator does not lead to a not sample-efficient algorithm (even without ontologies),
the second one does. 
This leaves open whether or not implemented systems
based on refinement operators admit generalization guarantees, as they
implement complex heuristics and optimizations.

As our final contribution we present SPELL, a SAT-based system that
implements bounded fitting of \EL-concepts under \ELHdr-ontologies. 
We evaluate SPELL on several datasets and compare it to
the only other available learning system for \EL that we are aware of,
the \emph{\EL tree learner (ELTL)} incarnation of the
\emph{DL-Learner} system~\cite{DBLP:journals/ws/BuhmannLW16}. 
We find that the running time of SPELL is almost always significantly
lower than that of ELTL. Since, as we also show, it is the size of the
target concept that has most impact on the running time, this means that
SPELL can learn larger target queries than ELTL.  We also analyze the
relative strengths and weaknesses of the two approaches, identifying
classes of inputs on which one of the systems performs significantly
better than the other one. Finally, we make initial experiments
regarding generalization, where both systems generalize well to unseen
data, even on very small samples. While this is expected for SPELL,
for ELTL it may be due to the fact that some of the heuristics prefer
fittings of small size, which might make ELTL an Occam algorithm.

  Proof details are provided in the appendix.

\paragraph{Related Work} Cohen and Hirsh identified a fragment of the
early DL \mn{CLASSIC} that admits sample-efficient PAC
learning, even in polynomial time~\cite{DBLP:journals/ml/CohenH94}.
For several DLs such as \EL and \mn{CLASSIC}, concepts are learnable
in polynomial time in Angluin's framework of exact learning with
membership and equivalence
queries~\cite{DBLP:journals/ml/FrazierP96,DBLP:conf/icdt/CateD21,FJL-IJCAI21,FJL-IJCAI22}.
The algorithms can be transformed in a standard way into
sample-efficient polynomial time PAC learning algorithms that,
however,
additionally use
membership queries to an oracle~\cite{DBLP:journals/ml/Angluin87}.
It is known that sample-efficient PAC learning under certain
assumptions implies the existence of Occam
algorithms~\cite{DBLP:journals/tcs/BoardP92}. These assumptions,
however, 
do not apply to the learning tasks studied here.
%
%
%
%
%
%
%
%

\section{Preliminaries}

\paragraph{Concepts, Ontologies, Queries.}
Let \NC, \NR, and \NI be countably infinite sets of \emph{concept
names}, \emph{role names}, and \emph{individual names}, respectively.
An \emph{\EL-concept} is formed according to the syntax rule \[ C,D
::= \top \mid A \mid C \sqcap D \mid \exists r . C\] where $A$ ranges
over \NC and $r$ over \NR. A concept of the form $\exists r . C$ is
called an \emph{existential restriction} and the \emph{quantifier
depth} of a concept is the maximum nesting depth of existential
restrictions in it. An \emph{\ELHdr-ontology}~\Omc is a finite set of
\emph{concept inclusions (CIs)} $C \sqsubseteq D$, \emph{role
inclusions} $r \sqsubseteq s$, and \emph{range assertions}
$\mn{ran}(r) \sqsubseteq C$ where $C$ and~$D$ range over \EL-concepts
and $r,s$ over role names. An \emph{\EL-ontology} is an
\ELHdr-ontology that uses neither role inclusions nor range assertions.
We also sometimes mention \emph{\ELI-concepts} and
\emph{\ELI-ontologies}, which extend their \EL-counterparts with
inverse roles $r^-$ that can be used in place of role names. 
See \cite{DL-Textbook} for more information.
%
A \emph{database} \Dmc (also called \emph{ABox} in a DL context) is a
finite set of \emph{concept assertions} $A(a)$ and \emph{role
  assertions} $r(a,b)$ where $A \in \NC$, $r \in \NR$, and
$a,b \in \NI$.  We use $\mn{adom}(\Dmc)$ to denote the set of
individual names that are used in \Dmc. A \emph{signature} is a set of
concept and role names, in this context uniformly referred to as
\emph{symbols}. For any syntactic object~$O$, such as a concept or an
ontology, we use $\mn{sig}(O)$ to denote the set of symbols used in
$O$ and $||O||$ to denote the \emph{size} of $O$, that is, the number
of symbols used to write $O$ encoded as a word over a finite alphabet,
with each occurrence of a concept or role name contributing a single
symbol.

The semantics is defined in terms of \emph{interpretations}
$\Imc=(\Delta^\Imc,\cdot^\Imc)$ where $\Delta^\Imc$ is the
\emph{domain} of \Imc and $\cdot^\Imc$ assigns a set
$A^\Imc\subseteq \Delta^\Imc$ to every $A\in\NC$ and a binary relation
$r^\Imc\subseteq \Delta^\Imc\times\Delta^\Imc$ to every $r\in \NR$.
The \emph{extension} $C^\Imc$ of \EL-concepts $C$ is then defined as
usual \cite{DL-Textbook}.  An interpretation~\Imc \emph{satisfies} a
concept or role inclusion $\alpha \sqsubseteq \beta$ if
$\alpha^\Imc \subseteq \beta^\Imc$, a range assertion
$\mn{ran}(r) \sqsubseteq C$ if the projection of $r^\Imc$ to the
second component is contained in $C^\Imc$, a concept assertion $A(a)$
if $a \in A^\Imc$, and a role assertion $r(a,b)$ if $(a,b)\in r^\Imc$.
We say that \Imc is a \emph{model} of an ontology/database if it
satisfies all inclusions/assertions in it.

An \EL-concept $C$ can be viewed as an \emph{\EL-query (ELQ)}~$q$, as
follows.  Let $\Dmc$ be a database and \Omc an \ELHdr-ontology.  Then
$a\in \mn{adom}(\Dmc)$ is an \emph{answer} to $q$ on \Dmc w.r.t.\ \Omc
if $a\in C^{\Imc}$ for all models $\Imc$ of \Dmc and \Omc.
In a similar way, we may view \ELI-concepts as
\emph{\ELI-queries (ELIQs)}. We will from now on mostly view
\EL-concepts as ELQs. This does not, however, restrict their use,
which may be as actual  
queries or as concepts used as building blocks for ontologies.

\newcommand{\qcontains}{\sqsubseteq}

An \emph{ontology-mediated query (OMQ) language} is a pair
$(\Lmc,\Qmc)$ with \Lmc an ontology language and \Qmc a query
language, such as $(\ELHdr,\text{ELQ})$ and $(\ELI,\text{ELIQ})$.  For
a query language \Qmc and signature $\Sigma$, we use $\Qmc_{\Sigma}$ to
denote the set of all queries $q \in \Qmc$ with $\mn{sig}(q) \subseteq
\Sigma$.
All query languages considered in this paper are unary, that is, they
return a subset of $\mn{adom}(\Dmc)$ as answers. We use
$q(\Dmc \cup \Omc)$ to denote the set of answers to $q$ on \Dmc
w.r.t.\ \Omc. For an \Lmc-ontology \Omc and queries $q_1,q_2$, we
write $\Omc \models q_1 \qcontains q_2$ if for all databases \Dmc,
$q_1(\Dmc \cup \Omc) \subseteq q_2(\Dmc \cup \Omc)$.  We say that
$q_1$ and $q_2$ are \emph{equivalent} w.r.t.\ \Omc, written
$\Omc \models q_1 \equiv q_2$, if $\Omc \models q_1 \qcontains q_2$
and $\Omc \models q_2 \qcontains q_1$. When $\Omc = \emptyset$, we
write $q_1\sqsubseteq q_2$ and $q_1\equiv q_2$.

\enlargethispage*{1mm}
Every ELQ $q$ may be viewed as a database $\Dmc_q$ in an obvious way,
e.g.\ $q=\exists r . \exists s . A$ as
$\Dmc_q = \{ r(a_q,a_1),s(a_1,a_2),$ $A(a_2) \}$. Let $\Dmc_1,\Dmc_2$
be databases and $\Sigma$ a signature.  A \emph{$\Sigma$-simulation}
from $\Dmc_1$ to $\Dmc_2$ is a relation
$S \subseteq \mn{adom}(\Dmc_1) \times \mn{adom}(\Dmc_2)$ such that for
all $(a_1,a_2)\in S$:
  \begin{enumerate}

  \item if $A(a_1)\in \Dmc_1$ with $A \in \Sigma$, then
    $A(a_2) \in \Dmc_2$;

  \item if $r(a_1, b_1) \in \Dmc_1$ with $r \in \Sigma$, there is
    $r(a_2,b_2) \in \Dmc_2$ such that $(b_1, b_2) \in S$.

  \end{enumerate}
  
  For $a_1\in\mn{adom}(\Dmc_1)$ and $a_2\in \mn{adom}(\Dmc_2)$, we
  write $(\Dmc_1,a_1)\preceq_\Sigma (\Dmc_2,a_2)$ if there is a
  $\Sigma$-simulation $S$ from $\Dmc_1$ to $\Dmc_2$ with
  $(a_1,a_2)\in S$. We generally drop the mention of $\Sigma$ in case
  that $\Sigma=\NC\cup \NR$.
  The following
  well-known lemma links simulations to ELQs.
\begin{lemma}\label{lem:char-containment}
  For all ELQs $q$, databases $\Dmc$, and $a \in \mn{adom}(\Dmc)$:
  $a \in q(\Dmc)$ iff $(\Dmc_q,a_q)\preceq (\Dmc,a)$. Consequently, for
  all ELQs~$q,p$: $q \sqsubseteq p$ iff
  $(\Dmc_p,a_p)\preceq (\Dmc_q,a_q)$.
\end{lemma}

\paragraph{Fitting.} A \emph{pointed database} is a pair $(\Dmc,a)$
with \Dmc a database and $a\in\mn{adom}(\Dmc)$. A \emph{labeled data
  example} takes the form $(\Dmc,a,+)$ or $(\Dmc,a,-)$, the former
being a \emph{positive example} and the latter a \emph{negative
  example}. 

  Let \Omc be an ontology, \Qmc a query language, and $E$ a collection
  of labeled data examples. A query $q \in \Qmc$ \emph{fits} $E$
  w.r.t.\ \Omc if $a \in q(\Dmc\cup\Omc)$ for all $(\Dmc,a,+) \in E$
  and $a \notin q(\Dmc\cup\Omc)$ for all $(\Dmc,a,-) \in E$. We then
  call $E$ a \emph{$q$-labeled data example w.r.t.\ \Omc}. We say that
  $q$ is a \emph{most specific fitting} if
  $\Omc \models q \qcontains q'$ for every $q' \in \Qmc$ that fits
  $E$, and that it is \emph{most general} if
  $\Omc \models q' \qcontains q$ for every $q' \in \Qmc$ that fits
  $E$.
  \begin{example}
    \label{ex1}
    Consider the collection $E_0$ of examples 
    $( \{r(a,a),A(a),B(a)\},a,+), (\{A(a),r(a,b), B(b)\}, a,+),$
    $(\{r(a,b)\},b,-)$. It has several ELQ fittings, the most
    specific one being $A\sqcap \exists r.B$.  There is no most
    general fitting ELQ as both $A$ and $\exists r . B$ fit, but no
    common generalization does.
  \end{example}
  A \emph{fitting algorithm} for an OMQ language $(\Lmc,\Qmc)$ is an
  algorithm that takes as input an \Lmc-ontology \Omc and a collection
  of labeled data examples $E$ and returns a query $q \in \Qmc$ that
  fits $E$ w.r.t.\ \Omc, if such a $q$ exists, and otherwise reports
  non-existence or does not terminate. The \emph{size-restricted
    fitting problem} for $(\Lmc,\Qmc)$ means to decide, given
  a collection of labeled data examples $E$, an \Lmc-ontology \Omc,
  and an $s \geq 1$ in unary, whether there is a query $q \in \Qmc$
  with $||q|| \leq s$ that fits $E$
  w.r.t.~\Omc. 

\smallskip

It is well-known that for every database~\Dmc and
\ELHdr-ontology~\Omc, we can compute in polynomial time a
database~$\Umc_{\Dmc,\Omc}$ that is \emph{universal for ELQs} in the
sense that $a\in q(\Dmc\cup \Omc)$ iff $a\in q(\Umc_{\Dmc,\Omc})$ for
all ELQs $q$ and $a\in\mn{adom}(\Dmc)$~\cite{DBLP:conf/ijcai/LutzTW09}.
Given a collection of labeled data examples $E$ and an \ELHdr-ontology
\Omc, we denote with $E_\Omc$ the collection 
obtained from $E$ by replacing each (positive or negative) example
$(\Dmc,a,\cdot)$ with $(\Umc_{\Dmc,\Omc},a,\cdot)$.  The following proposition shows
that a fitting algorithm for ELQ without ontologies also
gives rise to a fitting algorithm for $(\ELHdr,\text{ELQ})$ with at
most a polynomial increase in running time. It is immediate from the
definition of universality.
\begin{prop}
  \label{prop:tbox} An ELQ $q$ fits a collection of labeled
  examples $E$ w.r.t.\ an \ELHdr-ontology \Omc iff $q$ fits $E_\Omc$
  w.r.t.\ $\emptyset$.
\end{prop}
We remark that in contrast to ELQs, finite databases that are
universal for ELIQs need not exist
\cite{DBLP:conf/dlog/FunkJL22}.

\paragraph{PAC Learning.} 
We recall the definition of PAC learning, in a formulation
that is tailored towards OMQ languages.
Let $P$ be a probability distribution over pointed databases and let
$q_T$ and $q_H$ be queries, the target and the hypothesis. The
error of $q_H$ relative to $q_T$ and $P$ is
$$
\mn{error}_{P,q_T}(q_H)=\mathop{\mn{Pr}}_{(\Dmc,a)\sim P}\!(a \in
q_H(\Dmc\cup\Omc)
\mathrel{\Delta} q_T(\Dmc\cup\Omc))$$ where
$\Delta$ denotes symmetric difference and
$\mathop{\mn{Pr}}_{(\Dmc,a)\sim P} X$
is the probability of $X$ when drawing $(\Dmc,a)$
randomly according to $P$. 

    

%
\begin{definition}
  \label{def:efflearn}
  A \emph{PAC learning algorithm} for an OMQ language $(\Lmc,\Qmc)$ is
  a (potentially randomized) algorithm $\Amf$ associated with a
  function
  $m: \mathbb{R}^2 \times \mathbb{N}^4 \rightarrow \mathbb{N}$ such
  that
  \begin{itemize}

  \item $\Amf$ takes as input an \Lmc-ontology \Omc and a collection of
    labeled data examples $E$;

  \item for all $\epsilon,\delta \in (0,1)$, all \Lmc-ontologies~\Omc,
    all finite signatures $\Sigma$, all $s_Q,s_E \geq 0$, all
    probability distributions $P$ over pointed databases $(\Dmc,c)$
    with $\mn{sig}(\Dmc) \subseteq \Sigma$ and $||\Dmc|| \leq s_E$,
    and all $q_T \in \Qmc_\Sigma$ with $||q_T|| \leq s_Q$, the
    following holds: when running $\Amf$ on \Omc and a collection $E$ of
    at least $m(1/\delta,1/\epsilon,||\Omc||,|\Sigma|,s_Q,s_E)$
    labeled data examples that are $q_T$-labeled w.r.t.\ \Omc and
    drawn according to $P$, it returns a hypothesis $q_H$ such that with
    probability at least $1 - \delta$ (over the choice of $E$), we
    have $\mn{error}_{P,q_T}(q_H) \leq \epsilon$.
    
\end{itemize}
We say that $\Amf$ \emph{has sample size} $m$ and call $\Amf$
\emph{sample-efficient} if $m$ is a polynomial. 
\end{definition}
%
%
%
Note that a PAC learning algorithm is not required to  
terminate if no fitting query exists.  
It would be desirable to even attain \emph{efficient} PAC learning
which additionally requires $\Amf$ to be a polynomial time
algorithm. However, ELQs are known to not be efficiently PAC learnable
even without ontologies, unless $\RPclass=\NPclass$
\cite{DBLP:conf/ecml/Kietz93,DBLP:journals/corr/abs-2208-10255}.  The
same is true for ELIQs and any other class of conjunctive
queries that contains all ELQs.

\section{Bounded Fitting and Generalization}
\label{sect:fitvslearn}

We introduce bounded fitting and analyze when fitting
algorithms are PAC learning algorithms.
%
\begin{definition}
  \label{def:boundedfit}
  Let $(\Lmc,\Qmc)$ be an OMQ language and let $\Amf$ be an algorithm for
  the size-restricted fitting problem for $(\Lmc,\Qmc)$. Then
  $\textnormal{\textsc{Bounded-Fitting}}_\Amf$ is the algorithm that, given a
  collection of labeled data examples $E$ and an \Lmc-ontology~\Omc,
  runs $\Amf$ with input $(E,\Omc,s)$ to decide whether there is a
  $q \in \Qmc$ with $||q|| \leq s$ that fits $E$ w.r.t.~\Omc, for
  $s=1, 2, 3\ldots$, returning a fitting query as soon as it finds one.
\end{definition}
%
%
%
\begin{example}
Consider again Example~\ref{ex1}. For $s=1$, bounded
  fitting tries the candidates $\top,A,B,\exists r.\top$ and returns the
  fitting~$A$. If started on $E_0$ extended with $(\{A(a)\},a,-)$, it
  finds one of the fitting ELQs $A\sqcap\exists r.\top$ and
  $\exists r . B$ in Round~$2$.
\end{example}

In spirit, bounded fitting focusses on finding fitting queries when
they exist, and not on deciding the existence of a fitting query. This
is in analogy with bounded model checking, which focusses on finding
counterexamples rather than on proving that no such examples exist. If
an upper bound on the size of fitting
queries 
is known, however, we can make bounded fitting terminate by reporting
non-existence of a fitting query once the bound is exceeded.  
This is more of theoretical than of practical interest since
the size bounds tend to be large. For ELQs without ontologies and for
$(\EL,\text{ELQ})$, for instance, it is double
exponential~\cite{mauricemaster}. It thus seems more realistic to run
an algorithm that decides the existence of a  
fitting in parallel to bounded fitting and to report the
result as soon as one of the algorithms terminates. There are also
important cases where fitting existence is undecidable, such as for
the OMQ language
$(\ELI,\text{ELIQ})$~\cite{DBLP:conf/ijcai/FunkJLPW19}. Bounded
fitting may be used also in such cases as long as the size-restricted
fitting problem is still decidable. This is the case for
$(\ELI,\text{ELIQ})$, as a direct consequence of query evaluation to
be decidable in this OMQ language~\cite{DBLP:conf/owled/BaaderLB08},
see Appendix H.



A major advantage of bounded fitting is that it yields a
sample-efficient PAC learning algorithm with sample size linear in the
size of the target query. This is because bounded fitting is an Occam
algorithm which essentially means that it produces a fitting query
that is at most polynomially larger than the fitting query of minimal
size \cite{DBLP:journals/jacm/BlumerEHW89}.\footnote{A precise
  definition of Occam algorithms is based on the notion of
VC-dimension; it is not crucial to the main part of the paper, details
can be found in the appendix.}  
%
\begin{restatable}{theorem}{thmbfgen}
  \label{thm:bfgen}
  Let $(\Lmc,\Qmc)$ be an OMQ language.
  \label{thm:boundedfitispac}
  Every bounded fitting algorithm for $(\Lmc,\Qmc)$ 
  is a
  (sample-efficient) PAC learning algorithm with sample size
  \mbox{$O\big (\frac{1}{\epsilon} \cdot \log \big (\frac{1}{\epsilon}
    \big )
  \cdot \log \big (\frac{1}{\delta} \big
  ) \cdot
   \log |\Sigma| \cdot
    ||q_T|| \big )$}. 
\end{restatable}
We remark that bounded fitting is \emph{robust} in that other natural
measures of query size (such as the number of existential
restrictions) and enumeration sequences such as
$s=1,2,4,8,\dots$ also lead to sample-efficient PAC learning
algorithms. This results in some flexibility in implementations.

\smallskip

We next show that many other fitting algorithms are not sample-efficient
when used as PAC learning algorithms. We start with algorithms that
return fittings which are most specific or most general or of minimum
quantifier depth. No such algorithm is a sample-efficient PAC learning
algorithm, even without ontologies. 
%
\begin{theorem}
  \label{thm:notefficient}




%
  If $\Amf$ is a fitting algorithm for ELQs that satisfies one of the
  conditions below, then $\Amf$ is not a
  sample-efficient PAC learning algorithm. 
  \begin{enumerate}

  \item $\Amf$ always produces a most specific 
  fitting, if it exists; 

  \item $\Amf$ always produces a most general
  fitting, if it exists; 

\item $\Amf$ produces a fitting of minimal
  quantifier depth, if a fitting exists.

\end{enumerate}
%
\end{theorem}

The proof of Theorem~\ref{thm:notefficient} relies on duals of finite
relational structures, which are widely known in the form of
homomorphism duals~\cite{DBLP:journals/jct/NesetrilT00}. Here, we introduce
the new notion of \emph{simulation} duals.

  %
Let $(\Dmc,a)$ be a pointed database and $\Sigma$ a signature.  A set
$M$ of pointed databases is a \emph{$\Sigma$-simulation dual} of
$(\Dmc,a)$ if for all pointed databases $(\Dmc', a')$, the following
holds:
$$ \begin{array}{rcl} (\Dmc, a) \preceq_\Sigma (\Dmc', a') &
  \text{iff} & (\Dmc', a') \not\preceq_\Sigma (\Dmc'', a'') \\[1mm] &&
  \quad\quad \text{for all } (\Dmc'',a'') \in M.  \end{array} $$
For illustration, consider the simulation dual $M$ of $(\Dmc_q,a_q)$
for an ELQ $q$. Then every negative example for $q$ has a simulation
into an element of $M$ and $q$ is the most general ELQ that fits
$\{(\Dmc,a,-)\mid (\Dmc,a)\in M\}$.  We exploit this in the proof of
Theorem~\ref{thm:notefficient}. Moreover, we rely on the fact that
ELQs have simulation duals of polynomial size. In contrast,
(non-pointed) homomorphism duals of tree-shaped databases may become
exponentially large~\cite{DBLP:journals/siamdm/NesetrilT05}.

\begin{restatable}{theorem}{thmduals} \label{thm:duals} Given an ELQ
  $q$ and a finite signature
  $\Sigma$, a $\Sigma$-simulation dual $M$ of $(\Dmc_q,a_q)$
  of size
 $||M||\leq 3\cdot |\Sigma| \cdot ||q||^2$
  can be
  computed in polynomial time.
  Moreover, if $\Dmc_q$ contains only a single $\Sigma$-assertion 
  that mentions~$a_q$, then $M$ is a singleton.
\end{restatable}
The notion of simulation duals is of independent
interest and we develop it further in the appendix.  We show that
Theorem~\ref{thm:duals} generalizes from databases $\Dmc_q$ to all
pointed databases $(\Dmc,a)$ such that the directed graph induced by
the restriction of $\Dmc$ to the individuals reachable (in a directed
sense) from $a$ is a DAG.  Conversely, databases that are not of this
form do not have finite simulation duals. We find it interesting to
recall that DAG-shaped databases do in general not have finite homomorphism
duals~\cite{DBLP:journals/jct/NesetrilT00}.

%
%
%
%
%
\smallskip

Using Theorem~\ref{thm:duals}, we now prove Point~2 of
Theorem~\ref{thm:notefficient}. Points~1 and~3 are proved in
the appendix.
\begin{proof}
  To highlight the intuitions, we leave out some minor technical
  details that are provided in the appendix.  Assume to the contrary
  of what we aim to show that there is a sample-efficient PAC learning
  algorithm that produces a most general fitting ELQ, if it exists,
  with associated polynomial function
  $m \colon \mathbb{R}^2\times \mathbb{N}^4$ as in
  Definition~\ref{def:efflearn}. As target ELQs $q_T$, we use concepts
  $C_i$ where $C_0 = \top$ and
  $C_i = \exists r. (A \sqcap B \sqcap C_{i - 1})$. Thus, $C_i$ is an
  $r$-path of length~$i$ in which every non-root node is labeled with
  $A$ and~$B$.

  Choose $\Sigma = \{A, B, r\}$, $\delta = \epsilon = 0.5$, and $n$
  large enough so that
  $2^n > 2m(1/\delta,1/\epsilon,0,|\Sigma|,3n, 3 \cdot |\Sigma| \cdot
  ||C_n||^2)$. 
  Further
  choose $q_T=C_n$. 

    We next construct negative examples; positive examples are not
    used. Define a set of ELQs $S= S_n$ where
    $$
         S_0 = \{ \top\}\ \ \ S_i = \{ \exists r. (\alpha \sqcap
         C) \mid C \in S_{i - 1}, \alpha \in \{A, B\}\}.
         $$
  %
  %
    %
         Note that the ELQs in $S$ resemble $q_T$ except that every
         node is labeled with only one of the concept names $A,B$.
         Now consider any $q\in S$.  Clearly, $q_T \sqsubseteq q$.
         Moreover, the pointed database $(\Dmc_q,a_q)$ contains a single assertion that mentions $a_q$. By
         Theorem~\ref{thm:duals}, $q$ has a singleton
         $\Sigma$-simulation dual $\{(\Dmc_q',a_q')\}$ with
         $||\Dmc_q'|| \leq 3 \cdot |\Sigma| \cdot ||C_n||^2$. We shall
         use these duals as negative examples.

         The two crucial properties of $S$ are that for all
         $q \in S$,
    \begin{enumerate}
      
    \item $q$ is the most general ELQ that fits $(\Dmc'_q, a'_q,-)$;

    \item for all $T \subseteq S$, $q \notin T$ implies
      $\bigsqcap_{p \in T} p \not\sqsubseteq q$.

    \end{enumerate}
    %
    By Point~1 and since $q_T \sqsubseteq q$, each $(\Dmc_q',a_q')$
    is also a
    negative example for $q_T$.
   
    Let the probability distribution $P$ assign probability
    $\frac{1}{2^n}$ to all $(\Dmc'_q, a'_q)$ with $q \in S$ and
    probability $0$ to all other pointed databases.
    Now assume that the algorithm is started on a collection of
    $m(1/\delta,1/\epsilon,0,|\Sigma|,3n, 3 \cdot |\Sigma| \cdot
    ||C_n||^2)$ labeled data examples $E$ drawn according
    to~$P$. 
%
It follows from Point~1 that $q_H = \bigsqcap_{(\Dmc'_q, a'_q) \in E}
q$ is the most general ELQ that fits $E$. Thus, (an ELQ equivalent to)
$q_H$ is output by the algorithm.

To obtain a contradiction, it suffices to show that with probability
$1 - \delta$, we have $\mn{error}_{P, q_T}(q_H) > \epsilon$. We
argue that, in fact, $q_H$ violates all (negative) data examples that are
not in the sample~$E$, that is, $a_q \in q_H(\Dmc_p)$ for all
$p \in S$ with $(\Dmc_p', a_p') \notin E$.
%
The definition of $P$
and choice of $n$ then yield that with probability~1,
$\mn{error}_{P, q_T}(q_H) = \frac{|S| - |E|}{|S|} > \frac{1}{2}$.

Thus consider any $p \in S$ such that $(\Dmc_p', a_p') \notin E$.
It follows from Point~2 that $q_H \not\sqsubseteq p$ and the definition
of duals may now be used to derive $a_p' \in q_H(\Dmc_p')$ as desired.
%
%
%
%
\end{proof}

\section{Refinement Operators}

We discuss fitting algorithms based on refinement operators, used in
implemented systems such as ELTL, and show that the generalization
abilities of such algorithms subtly depend on the exact operator (and
strategy) used.

Let $(\Lmc,\Qmc)$ be an OMQ language. A \emph{(downward) refinement}
of a query $q \in \Qmc$ w.r.t.\ an \Lmc-ontology $\Omc$ is any
$p \in \Qmc$ such that $\Omc \models p \qcontains q$ and
$\Omc\not\models q \qcontains p$. A \emph{(downward) refinement
  operator} for $(\Lmc,\Qmc)$ is a function $\rho$ that associates
every $q \in Q_\Sigma$, \Lmc-ontology \Omc, and finite signature
$\Sigma$ with a set $\rho(q,\Omc,\Sigma)$ of downward refinements
$p \in \Qmc_\Sigma$ of $q$ w.r.t.~\Omc. The operator $\rho$ is
\emph{ideal} if it is finite and complete where $\rho$ is
\begin{enumerate}

\item \emph{finite} if $\rho(q,\Omc,\Sigma)$ is finite for all $q$,
  \Omc, and finite $\Sigma$, and

\item \emph{complete} if for all finite signatures $\Sigma$ and all
  $q,p \in \Qmc_\Sigma$, $\Omc \models p \qcontains q$ implies that
  there is a finite \emph{$\rho,\Omc,\Sigma$-refinement sequence} from
  $q$ to $p$, that is, a sequence of queries $q_1,\dots,q_n$ such that
  $q=q_1$, $q_{i+1} \in \rho(q_i,\Omc,\Sigma)$ for $1 \leq i < n$, and
  $\Omc \models q_n \equiv p$.

\end{enumerate}
When \Omc is empty, we write $\rho(q,\Sigma)$ in place of
$\rho(q,\Omc,\Sigma)$.


For $(\EL,\text{ELQ})$ and thus also for $(\ELHdr,\text{ELQ})$,
it is known that no ideal refinement operator
exists~\cite{KriegelPhD}. This problem can be overcome by making use
of Proposition~\ref{prop:tbox} and employing an ideal refinement
operator for ELQs without ontologies, which does
exist~\cite{DBLP:conf/ilp/LehmannH09}. But also these refinement
operators are not without problems.  It was observed in
\cite{DBLP:conf/dlog/Kriegel21} that for any such operator,
non-elementarily long refinement sequences exist, potentially
impairing the practical use of such operators. We somewhat relativize
this by the following observation. A refinement operator $\rho$ for
$(\Lmc,\Qmc)$ is \emph{$f$-depth bounded}, for
$f: \mathbb{N} \rightarrow \mathbb{N}$, if for all $q,p \in \Qmc$ and
all \Lmc-ontologies \Omc with $\Omc \models p \sqsubseteq q$, there
exists a $\rho,\Omc,\Sigma$-refinement sequence from $q$ to $p$ that
is of length at most $f(||p||)$.
\begin{theorem}
  \label{thm:expdepthbound}
  Let $(\Lmc,\Qmc)$ be an OMQ-language. If $(\Lmc,\Qmc)$ has an
  ideal refinement operator, then it has a $2^{O(n)}$-depth bounded
  ideal refinement operator.
\end{theorem}
The depth bounded operator in Theorem~\ref{thm:expdepthbound} is
obtained by starting with some operator $\rho$ and adding to each
$\rho(q,\Omc,\Sigma)$ all $p \in \Qmc_\Sigma$ such that
$\Omc \models p \qcontains q$, $\Omc\not\models q \qcontains p$, and
$||p|| \leq ||q||$. Note that the size of queries is used in an
essential way, as in Occam algorithms.

\smallskip
A refinement operator by itself is not a fitting algorithm 
as one also needs a strategy for applying the operator.  We 
use breadth-first search as a simple yet natural such strategy.

\newcommand{\numat}[1]{||#1||}

%
We consider two related refinement operators $\rho_1$ and $\rho_2$ for
ELQs. The definition of both operators refers to (small) query size,
inspired by Occam algorithms.
%
%
Let~$q$ be an ELQ. Then $\rho_1(q,\Sigma)$ is the set of all
$p \in \text{ELQ}_\Sigma$ such that $p \sqsubseteq q$,
$q \not\sqsubseteq p$, and $\numat{p} \leq 2\numat{q}+1$. The operator
$\rho_2$ is defined like $\rho_1$ except that we include in
$\rho_2(q,\Sigma)$ only ELQs
$p$ that are a \emph{(downward) neighbor} of $q$,
that is, for all ELQs~$p'$, $p \qcontains p' \qcontains q$ implies
$p' \qcontains p$ or $q \qcontains p'$. The following lemma shows that
$\rho_2(q,\Sigma)$ actually contains \emph{all} neighbors of $q$ with
$\mn{sig}(q) \subseteq \Sigma$, up to
equivalence. 
An ELQ $q$ is \emph{minimal} if there is no ELQ $p$ such that
$\numat{p} < \numat{q}$ and $p \equiv q$.
%
\begin{restatable}{lemma}{lemsmallneighbor}
    \label{lem:smallneighbor}
  For every ELQ $q$ and minimal downward neighbor $p$ of $q$, we have
  $\numat{p} \leq 2\numat{q}+1$.
\end{restatable}
%
Both $\rho_1$ and $\rho_2$ can be computed by brute force. For more
elaborate approaches to computing $\rho_2$, see
\cite{DBLP:conf/dlog/Kriegel21} 
where
downward neighbors of ELQs are studied in detail.
\begin{restatable}{lemma}{lemrhosideal}
  \label{lem:rhosideal}
  $\rho_1$ and $\rho_2$ are ideal refinement operators for ELQ.
\end{restatable}
We next give more details on what we mean by breadth-first search.
Started on a collection of labeled data examples $E$, the algorithm
maintains a set $M$ of candidate ELQs that fit all positive examples
$E^+$ in $E$, beginning with $M = \{\top\}$ and proceeding in
rounds. If any ELQ $q$ in $M$ fits $E$, then we return such a fitting $q$ with
$||q||$ smallest. Otherwise, the current set $M$ is replaced with the
set of all ELQs from $\bigcup_{q \in M} \rho(q,\text{sig}(E))$ that
fit~$E^+$, and the next round begins. For $i \in \{1,2\}$, let
$\Amf_i$ be the version of this algorithm that uses refinement
operator $\rho_i$.
Although $\rho_1$
and $\rho_2$ are defined quite similarly, the behavior of the algorithms $\Amf_1$
and $\Amf_2$ differs.
\begin{restatable}{theorem}{thmaonepac}
  \label{thm:aonepac}
  $\Amf_1$ is a sample-efficient PAC learning algorithm, but $\Amf_2$ is not.
\end{restatable}
To prove Theorem~\ref{thm:aonepac}, we show that $\Amf_1$ is an Occam
algorithm while $\Amf_2$ produces a most general fitting (if it exists),
which allows us to apply Theorem~\ref{thm:notefficient}.
%

\medskip The above is intended to provide a case study of refinement
operators and their generalization abilities. Implemented
systems use refinement operators and strategies that are more complex
and include heuristics and optimizations. This makes it difficult to
analyze whether implemented refinement-based systems constitute a
sample-efficient PAC learner.

We comment on the ELTL system that we use in our experiments.  ELTL is
based on the refinement operator for $(\ELHdr,\text{ELQ})$ presented
in~\cite{DBLP:conf/ilp/LehmannH09}. That operator, however, admits
only \ELHdr ontologies of a rather restricted form: all CIs must be of
the form $A \sqsubseteq B$ with $A,B$ concept \emph{names}. Since no
ideal refinement operators for unrestricted $(\EL,\text{ELQ})$ exist
and ELTL does not eliminate ontologies in the spirit of
Proposition~\ref{prop:tbox}, it remains unclear whether and how ELTL
achieves completeness (i.e., finding a fitting whenever there is one).

\begin{figure*}[t!]
  \centering 
  \resizebox{\textwidth}{!}{\input{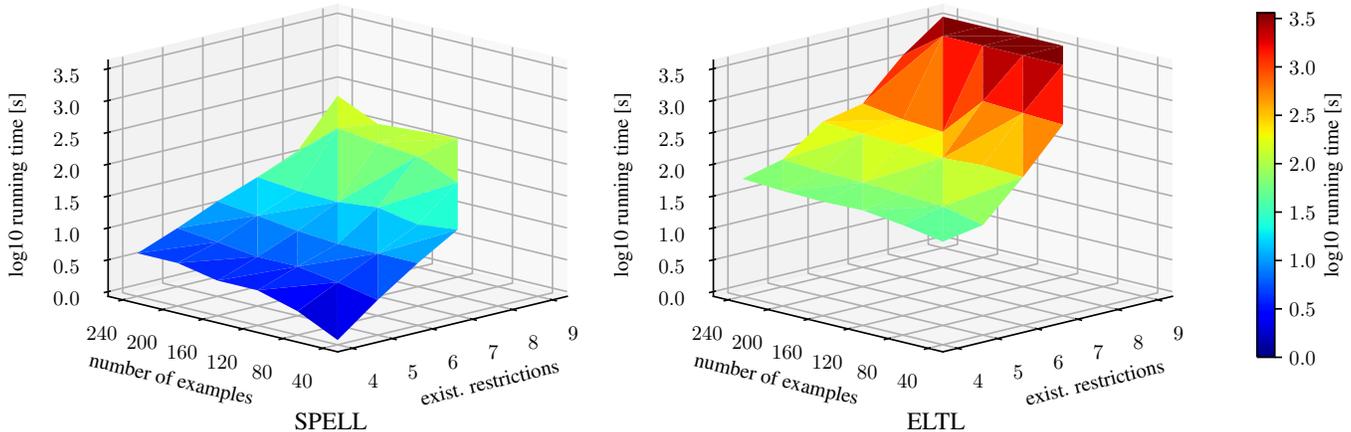}}
  \caption{Yago experiment, dark red area indicates timeout (60min)}
  \label{fig:yago-running-time}
  \vspace*{-3mm}
\end{figure*}

\section{The \systemname System}
\label{sect:elpl}

We implemented bounded fitting for the OMQ language
$(\ELHdr, \text{ELQ})$ in the system \systemname (for \emph{SAT-based
  PAC \EL concept Learner}).\footnote{Available at
  \url{https://github.com/spell-system/SPELL}.}
  \systemname takes as input a knowledge base in
OWL RDF/XML format that contains both an \ELHdr ontology $\Omc$ and a
collection $E$ of positive and negative examples, and it outputs an
ELQ represented as a SPARQL query.  \systemname is implemented in
Python~3 and uses the PySat
library 
to interact with the
Glucose SAT
solver. 
It provides integration into the SML-Bench benchmark
framework~\cite{DBLP:journals/semweb/WestphalBBJL19}. 

In the first step, \systemname removes the ontology \Omc by replacing
the given examples $E$ with $E_\Omc$ as per
Proposition~\ref{prop:tbox}. It then runs bounded fitting in the
variant where in each round~$n$, fitting ELQs with at most $n-1$
existential restrictions are considered (rather than fitting ELQs $q$ 
with $||q|| \leq n$). The existence of such a fitting is checked using
the SAT solver.
Also this variant of bounded
fitting results in a 
sample-efficient PAC learning algorithm, with sample size
$O\big (\frac{1}{\epsilon} \cdot \log \big (\frac{1}{\epsilon} \big )
\cdot \log \big (\frac{1}{\delta} \big ) \cdot |\Sigma| \cdot ||q_T||
\big)$, see the appendix. We prefer this variant for implementation 
because it admits a more natural reduction to SAT, described
next. 

From $E_\Omc$ and the bound $n$, we construct a propositional formula
$\varphi = \varphi_1 \land \varphi_2$ that is satisfiable if and only
if there is an ELQ $q$ over $\Sigma=\text{sig}(E_\Omc)$ with at most
$n - 1$ existential restrictions that fits $E_\Omc$. Indeed, any model
of $\varphi$ returned by the SAT solver uniquely represents a fitting
ELQ $q$.  More precisely, $\vp_1$ ensures that such a  model represents
\EL-concepts $C_1,\dots,C_n$ 
where each $C_i$ only contains existential restrictions of the form
$\exists r.  C_j$ with $j>i$, and we take $q$ to be $C_1$.  We use
variables of the form $c_{i, A}$ to express that the concept name $A$
is a conjunct of $C_i$, and variables  $x_{j,r}$ and $y_{i,j}$ to
express that $\exists r. C_j$ is a conjunct of~$C_i$.  Then
$\varphi_2$ enforces that the represented ELQ fits $E_\Omc$.  Let
$\Dmc$ be the disjoint union of all databases that occur in an example
in $E_\Omc$. We use variables $s_{i, a}$, with $1\leq i \leq n$ and
$a \in \mn{adom}(\Dmc)$, to express that
$a \in C_i(\Dmc)$; the exact definition of $\varphi_2$ uses
simulations and relies on Lemma~\ref{lem:char-containment}.
The number of variables in $\varphi$
is $O\big(n^2 \cdot |\Dmc|\big)$, thus linear
in $|\Dmc|$.

We have implemented several improvements over this basic reduction
of which we describe two. The first
improvement is based on the simple observation that for computing a
fitting ELQ with $n - 1$ existential restrictions, for every example
$(\Dmc',a,\pm)\in E_\Omc$ it suffices to consider individuals
that can be reached via at most $n - 1$ role assertions from~$a$.
Moreover, we may restrict $\Sigma$ to symbols that occur in all
$n-1$-reachable parts of the positive examples.
The second improvement is based on the
observation that the search space for satisfying assignments of
$\varphi$ contains significant \emph{symmetries} as the same ELQ $q$
may be encoded by many different arrangements of concepts
$C_1, \ldots C_n$.  We add constraints to $\vp$ so that the number of
possible arrangements is reduced, breaking many
symmetries. For details see the appendix.

\section{Experimental Evaluation}


We evaluate SPELL on several benchmarks\footnote{Available at \url{https://github.com/spell-system/benchmarks}.} and compare it to the ELTL component of the DL-Learner
system~\cite{DBLP:journals/ws/BuhmannLW16}. Existing benchmarks do not
suit our purpose as they aim at learning concepts that are formulated in more
expressive DLs of the \ALC family. As a consequence, a fitting \EL
concept almost never
exists. 
This is the case, for example, in the often used Structured Machine
Learning Benchmark \cite{DBLP:journals/semweb/WestphalBBJL19}.  We
thus designed several new benchmarks leveraging various existing
knowledge bases, making sure that a fitting \EL concept always
exists. We hope that our benchmarks will provide a basis also for
future experimental evaluations of \EL learning systems.

\begin{table*}
  \centering\small 
  \begin{tabular}{rllllllllllllllll}
    \toprule 
    Sample Size & 5 & 10 & 15 & 20 & 25 & 30 & 35 & 40 & 45 & 50 & 55 & 60 & 65 & 70 & 75\\\midrule 
    ELTL &  0.77 & 0.78 & 0.85 & 0.85 & 0.86 & 0.89 & 0.90 & 0.96 & 0.96 & 0.96 & 0.96 & 0.98 & 0.98 & 0.98 & 0.98\\
    SPELL & 0.80 & 0.81 & 0.84 & 0.85 & 0.86 & 0.86 & 0.89 & 0.97 & 0.98 & 0.98 & 0.98 & 0.98 & 0.98 & 0.98 & 0.98\\
    \bottomrule 
  \end{tabular}
  \caption{Generalization experiment accuracies}
  \label{tab:generalizatio-yago-exp}
\end{table*}

\paragraph{Performance Evaluation.}
We carried out two experiments that aim at evaluating the performance
of SPELL. The main questions are: Which parameters have most impact on the
running time? And how does the running time compare to that of ELTL?

The first experiment uses the Yago 4 knowledge base which combines the
concept classes of schema.org with data from
Wikidata~\cite{DBLP:conf/esws/TanonWS20}.  The smallest version of
Yago~4 is still huge and contains over 40 million assertions. We
extracted a fragment of 12 million assertions
assertions that focusses on movies and famous persons.  We then
systematically vary the number of labeled examples and the size of the
target ELQs. The latter take the form
$C_n=\exists \text{actor}. \bigsqcap^n_{i = 1} r_i. \top$ where each
$r_i$ is a role name that represents a property of actors in Yago and
$n$ is increased to obtain larger queries.  The positive examples are
selected by querying Yago with $C_n$ and the negative examples by
querying Yago with generalizations of $C_n$. The results are presented
in Figure~\ref{fig:yago-running-time}. 
They show that the size of the target query has a strong impact on the
running time whereas the impact of the number of positive and negative
examples is much more modest. We also find that SPELL performs
${\sim}$1.5 orders of magnitude better than ELTL, meaning in
particular that it can handle larger target queries.

Since Yago has only a very restricted ontology that essentially
consists of inclusions $A \sqsubseteq B$ with $A,B$ concept
names, we complement the above experiment with a second one
based on OWL2Bench. OWL2Bench is a benchmark for ontology-mediated
querying that combines a database generator with a hand-crafted
ontology which extends the University Ontology Benchmark
\cite{DBLP:conf/semweb/SinghBM20,DBLP:conf/www/ZhouGHWB13}.  The
ontology is formulated in OWL 2 EL and we extracted its \ELHdr
fragment which uses all aspects of this DL and comprises 142 
concept names, 83 role names, and 173 concept inclusions. We use datasets that contain
2500-2600 individuals and 100-200 examples, generated as in the Yago
case. We designed 6 ELQs with 3-5 occurrences of concept and role
names and varying topology. The results are shown in
Table~\ref{tab:owl2bench-running-time}.
\begin{table}
  \centering
  \small
\begin{tabular}{rrrrrrr}
\toprule
& o2b-1 & o2b-2 & o2b-3 & o2b-4 & o2b-5 & o2b-6 \\ \midrule
ELTL & TO & TO & 274 & 580 & 28 & 152 \\
SPELL & $<1$ & $<1$ & $<1$ & $<1$ & $<1$ & $<1$\\
\bottomrule\end{tabular}
\caption{OWL2Bench running times [s], TO: $>$60min}
\label{tab:owl2bench-running-time}  \vspace*{-3mm}
\end{table}
The difference in running time is even more pronounced in this
experiment, with SPELL returning a fitting ELQ almost instantaneously
in all cases.\footnote{ELTL crashes on this benchmark unless one
  option (`useMinimizer') is switched off. We thus ran ELTL without
  useMinimizer.}

\paragraph{Strengths and Weaknesses.} In this experiment, we aim to
highlight the respective strengths and weaknesses of SPELL and ELTL
or, more generally, of bounded fitting versus refinement-operator
based approaches. We anticipated that the performance of bounded
fitting would be most affected by the number of existential
restrictions in the target query whereas the performance of refinement
would be most affected by the (unique) length of the sequence
$C_1,\dots,C_k$ such that $C_1 = \top$, $C_{i+1}$ is a downward
neighbor of $C_i$ for $1 \leq i < k$, and $C_k$ is the target
query. Let us call this the \emph{depth} of $C_k$. The number of
existential restrictions and depth are orthogonal parameters. In the
\emph{$k$-path} benchmark, we use target ELQs of the form
$\exists r^k . \top$, $k \geq 1$. These should be difficult for
bounded fitting when the number $k$ of existential restrictions gets
large, but easy for refinement as the depth of $\exists r^k . \top$ is
only $k$. In the \emph{$k$-1-conj} benchmark, we use ELQs of the form
$\exists r.\bigsqcap^k_{i = 1} A_i$, $k \geq 1$. These have only one
existential restriction and depth $2^k$. ELQs in the
\emph{$k$-2-conj} benchmark take the form
$\exists r. \exists r . \bigsqcap^k_{i = 1} A_i$ and even have depth
$2^{2^k}$ \cite{DBLP:conf/dlog/Kriegel21}. These should be difficult
for refinement when $k$ gets large, but easy for SPELL. There is no
ontology and we use only a single positive and a single negative
example, which are the target ELQ and its unique upwards neighbor
(defined in analogy with downwards neighbors). The results 
in Table~\ref{tab:synthetic-running-time} 
confirm our
expectations, with ELTL arguably degrading faster than SPELL.
\pagebreak
\begin{table}
  \centering
  \small
  \begin{tabular}{lrrrrrr}
  \toprule
  & \multicolumn{2}{c}{$k$-path} & \multicolumn{2}{c}{$k$-1-conj} & \multicolumn{2}{c}{$k$-2-conj}\\
  $k$ & {\scriptsize ELTL} & {\scriptsize \systemname} & {\scriptsize ELTL} & {\scriptsize \systemname} & {\scriptsize ELTL} & {\scriptsize \systemname}\\ \midrule
  4 & 1 & \textless 1 & 1 & \textless 1& 1 & \textless 1\\
  6 & 1 & \textless 1 & 2 & \textless 1& 394 & \textless 1\\
  8 & 1 & \textless 1 & 20 & \textless 1& TO & \textless 1\\
  10 & 1 & \textless 1 & TO & \textless 1& TO & \textless 1\\
  12 & 1 & 26 & TO & \textless 1& TO &\textless 1\\
  14 & 1 & 30 & TO & \textless 1& TO & \textless 1\\
  16 & 1 & 68 & TO & \textless 1& TO & \textless 1\\
  18 & 1 & TO & TO & \textless 1& TO & \textless 1\\
\bottomrule\end{tabular}
  \caption{Strengths/weaknesses running time [s], TO: $>$10min}
  \label{tab:synthetic-running-time}
    \vspace*{-3mm}
\end{table}

\paragraph{Generalization.}

We also performed initial experiments to evaluate how well the
constructed fittings generalize to unseen data. We again use the Yago
benchmark, but now split the examples into training data and testing
data (assuming a uniform probability distribution).
Table~\ref{tab:generalizatio-yago-exp} lists the median accuracies of
returned fittings (over 20 experiments) where the number of examples
in the training data ranges from 5 to 75. As expected, fittings
returned by SPELL generalize extremely well, even when the number of
training examples is remarkably small. To our surprise, ELTL exhibits
the same characteristics. This may be due to the fact that some
heuristics of ELTL prefer fittings of smaller size, which might
make ELTL an Occam algorithm. It would be interesting to carry out
more extensive experiments on this aspect. 


\section{Conclusion and Future Work}

We have introduced the bounded fitting paradigm along with the
SAT-based implementation SPELL for $(\ELHdr,\text{ELQ})$, with
competitive performance and formal generalization guarantees. A
natural next step is to extend SPELL to other DLs such as \ELI, \ALC,
or \ELU, both with and without ontologies. We expect that, in the case
without ontology, a SAT encoding of the size-restricted fitting
problem will often be possible. The case with ontology is more
challenging; e.g., size-restricted fitting is \ExpTime-complete for
$(\ELI,\text{ELIQ})$, see Appendix~H for additional
discussion. It is also interesting to investigate query languages
beyond DLs such as conjunctive queries (CQs). Note that the
size-restricted fitting problem for CQs is
$\Sigma^p_2$-complete~\cite{DBLP:journals/ngc/GottlobLS99} and thus
beyond SAT solvers; one could resort to using an ASP solver or to CQs
of bounded treewidth.

It would also be interesting to investigate settings in which input
examples may be labeled erroneously or according to a target query
formulated in different language than the query to be learned. In both
cases, one has to admit non-perfect fittings and the optimization
features of SAT solvers and Max-SAT solvers seem to be promising for
efficient implementation.

%
%
%
%
%
%
%
%
%

\section*{Acknowledgements}

Balder ten Cate is supported by the European Union’s Horizon 2020
research and innovation programme under grant MSCA-101031081 and
Carsten Lutz by the DFG Collaborative Research Center 1320 EASE
and by BMBF in DAAD project 57616814 (SECAI).

\bibliographystyle{named}
\bibliography{ijcai23}

\cleardoublepage

\appendix

\section{Additional Preliminaries}

We make precise how ELQs can be viewed as databases as announced in
Section~2 of the main paper. Formally, we inductively associate to every
$q$ a pointed database $(\Dmc_q,a_q)$ with $\Dmc_q$
tree-shaped (recall that a database is tree-shaped if the directed
graph $G_\Dmc = (\mn{adom}(\Dmc),\{(a,b)\mid r(a,b)\in \Dmc\})$ is a
tree), as follows: 
\begin{itemize}

  \item If $q=\top$, then $\Dmc_q$ contains the single fact
    $\top(a_q)$;\footnote{We allow facts of the form $\top(a)$ for convenience.}

  \item if $q=A$, then $\Dmc_q$ contains the single fact
    $A(a_q)$;

  \item if $q=q_1\sqcap q_2$, then $\Dmc_q$ is obtained from
    $(\Dmc_{q_1},d_{q_1}),(\Dmc_{q_2},d_{q_2})$ by first taking the disjoint
    union of $\Dmc_{q_1}$ and $\Dmc_{q_2}$ and then identifying $a_{q_1}$ and
    $a_{q_2}$ to $a_q$;

  \item if $q=\exists r.p$, then $\Dmc_q$ is obtained from
    $(\Dmc_{p},a_{p})$ by taking $\Dmc_q=\Dmc_p\cup\{r(a_q,a_p)\}$ for
    a fresh individual $a_q$.

\end{itemize}
Let $\Imc_1$ and $\Imc_2$ be interpretations. The \emph{direct
  product}
of $\Imc_1$ and $\Imc_2$, denoted $\Imc_1 \times \Imc_2$, is
the interpretation with domain $\Delta^{\Imc_1} \times
\Delta^{\Imc_2}$ and such that for all concept names $A$ and role
names
$r$:
$$
\begin{array}{rcl}
  A^{\Imc_1 \times \Imc_2} &=& A^{\Imc_1} \times A^{\Imc_2} \\[1mm]
  r^{\Imc_1 \times \Imc_2} &=& r^{\Imc_1} \times r^{\Imc_2}.
\end{array}
$$


\section{Occam Algorithms and Proof of Theorem~\ref{thm:bfgen}}
\label{app:occam}

Let \Omc be an ontology and \Qmc a class of queries. For a set of
pointed databases $S$, we say that \Qmc \emph{shatters} $S$
w.r.t.~\Omc if for every subset $S' \subseteq S$, there is a
$q \in \Qmc$ such that
$S'=\{ (\Dmc,a) \in S \mid a \in q(\Dmc \cup \Omc)\}$. The
\emph{VC-dimension} of $\Qmc$ w.r.t.~\Omc is the cardinality of the
largest set of pointed databases $S$ that is shattered by \Qmc w.r.t.\ \Omc.

    Let $(\Lmc,\Qmc)$ be an OMQ language 
    and $\Amf$ a fitting algorithm for $(\Lmc,\Qmc)$.
%
  For an
  \Lmc-ontology \Omc, a finite signature $\Sigma$, and
  $s,m \geq 1$, we use $\Hmc^\Amf(\Omc,\Sigma,s,m)$ to denote the
  set of all outputs that $\Amf$ makes when started on \Omc and a
  collection of $m$ data examples $E$ such that
  $\mn{sig}(E) \subseteq \Sigma$ and $E$ is $q_T$-labeled w.r.t.~\Omc
  according to some
  $q_T \in \Qmc_\Sigma$ with $||q_T|| \leq s$.  This is called an \emph{effective
    hypothesis space of $\Amf$}. We say that $\Amf$ is an \emph{Occam
    algorithm}
  if there exists a polynomial $p$ and a constant $\alpha \in [0,1)$
  such that for all \Lmc-ontologies~\Omc, finite signatures $\Sigma$,
  and $s,m \geq 1$, the VC-dimension of $\Hmc^\Amf(\Omc,\Sigma,s,m)$
  w.r.t.~\Omc is bounded by $p(s,|\Sigma|) \cdot m^\alpha$.

  Theorem~3.2.1 of \cite{DBLP:journals/jacm/BlumerEHW89} then implies
  the following.
  \begin{lemma}
    \label{lem:ourblumer}
    If $\Amf$ is an Occam algorithm with the VC-dimension of effective
    hypothesis spaces bounded by $p(s,|\Sigma|) \cdot m^\alpha$, then
    $\Amf$ is a PAC learning algorithm with sample size
  $$
       m(1/\epsilon,1/\delta,n) = \max \Big (\frac{4}{\epsilon} \log
       \frac{2}{\delta}, \Big (\frac{8p(s,|\Sigma|)}{\epsilon} \log \frac{13}{\epsilon}
       \Big )^{1/(1-\alpha)}\Big ).
       $$
     \end{lemma}
     There are certain difference between the setup used in this paper
     and the setup in \cite{DBLP:journals/jacm/BlumerEHW89}.  We
     comment on why we still obtain Lemma~\ref{lem:ourblumer} from
     Theorem~3.2.1 of \cite{DBLP:journals/jacm/BlumerEHW89}. The aim
     of \cite{DBLP:journals/jacm/BlumerEHW89} is to study the learning
     of \emph{concept classes} which are defined in a general way as a
     set \Cmc of \emph{concepts} $C \subseteq X$ where $X$ is a fixed
     set of \emph{examples}. Consequently, their definition of PAC
     algorithms refers to concept classes and, in contrast to
     Definition~\ref{def:efflearn}, does neither mention ontologies
     nor signatures. However, when fixing an \Lmc-ontology \Omc and
     signature $\Sigma$, we obtain an associated concept class
     $\Cmc_{\Omc,\Sigma}$ by taking $X$ to be the set of all pointed
     $\Sigma$-databases and each query $q \in \Qmc$ as the concept
     that consists of all pointed $\Sigma$-databases that are positive
     examples for $q$. Moreover, by simply fixing \Omc and $\Sigma$,
     any fitting algorithm $\Amf$ for $(\Lmc,\Qmc)$ turns into a learning
     algorithm for $\Cmc_{\Omc,\Sigma}$ in the sense of
     \cite{DBLP:journals/jacm/BlumerEHW89}. Here, `fixing' means that
     we promise to only run $\Amf$ on input ontology \Omc and collections
     of labeled data examples $E$ such that
     $\mn{sig}(E) \subseteq \Sigma$ and $E$ is $q_T$-labeled
     w.r.t.~\Omc according to some $q_T \in \Qmc_\Sigma$. The
     definition of Occam algorithms in
     \cite{DBLP:journals/jacm/BlumerEHW89} refers to effective
     hypothesis spaces $\Hmc^\Amf(s,m)$ and requires that their
     VC-dimension is bounded by $p(s)\cdot m^\alpha$ (where $||\Omc||$
     and $|\Sigma|$ are considered constants). If $\Amf$ is Occam in our
     sense, then $\Amf$ with \Omc and $\Sigma$ fixed is Occam in the
     sense of \cite{DBLP:journals/jacm/BlumerEHW89}.  Theorem~3.2.1 of
     that paper then gives that $\Amf$ with \Omc and $\Sigma$ fixed is a
     PAC learning algorithm for $\Cmc_{\Omc,\Sigma}$ with the bound
     stated in Lemma~\ref{lem:ourblumer}.

     We remark that the precondition of Theorem~3.2.1 in
     \cite{DBLP:journals/jacm/BlumerEHW89} actually demands that the
     algorithm runs in polynomial time, but an analysis of the proof
     shows that this assumption is not used. Then, by
     Definition~\ref{def:efflearn}, every fitting algorithm $\Amf$ that
     is a PAC learning algorithm when restricted to \Omc and $\Sigma$,
     for \emph{any} \Omc and $\Sigma$ and with the \emph{same
       function} $m$ describing the sample size, is a PAC learning
     algorithm for $(\Lmc,\Qmc)$.

     \smallskip A final small difference is that, in
     \cite{DBLP:journals/jacm/BlumerEHW89}, the function $m$ in the
     definition of PAC algorithms does not depend on the size of the
     examples. Our version is a standard variation and does not impair
     the application of Blumer's Theorem~3.2.1: to see this, it
     suffices to observe that we do not use this parameter in the
     definition of effective hypothesis spaces and thus our Occam
     algorithms (with fixed \Omc and $\Sigma$) are also Occam
     algorithms in the sense of Blumer. Moreover, every PAC algorithm
     in the sense of Blumer is a PAC algorithm in our
     sense. Intuitively, having the example size as a parameter in the
     function $m$ makes the lower bounds (results on algorithms not
     being non-sample PAC learners) stronger as it makes it impossible
     to use examples of excessive size. It is also more generous
     regarding the upper bounds (developing PAC algorithms), but we
     do not make use of that generosity.
     
\medskip     
     
\thmbfgen*
\begin{proof}
  Let $\Bmf=\textsc{Bounded-Fitting}_\Amf$ be a bounded fitting algorithm
  for $(\Lmc,\Qmc)$. Let \Omc be an \Lmc-ontology, $\Sigma$ a
  finite signature, and $s,m \geq 0$. We show that the VC-dimension
  of $\Hmc^\Bmf(\Omc,\Sigma,s,m)$ is at most $O(s \cdot \log |\Sigma|)$.

  It is immediate from Definition~\ref{def:boundedfit} that when
  started on \Omc and a collection of $m$ data examples $E$ such that
  $\mn{sig}(E) \subseteq \Sigma$ and $E$ is $q_T$-labeled w.r.t.\ \Omc
  according to some $q_T \in Q_\Sigma$ with $||q_T|| \leq s$, then $\Bmf$
  returns a fitting $q \in \Qmc$ for $E$ w.r.t.~\Omc whose size
  $||q||$ is smallest among all fitting queries. 
  Consequently, $\Hmc^\Bmf(\Omc,\Sigma,s,m)$ consists only of queries
  $q \in \Qmc$ with \mbox{$||q|| \leq s$}. 
  There are at most $(|\Sigma|+c+1)^s$ such queries for some
  constant\footnote{The number of symbols from the finite alphabet
    used to encode syntactic objects as a word from the definition of
    $||\cdot||$.}  $c$ and since $2^{|S|}$ queries are needed to
  shatter a set~$S$, the VC-dimension of $\Hmc^\Bmf(\Omc,\Sigma,s,m)$ is
  at most $\log((|\Sigma|+c+1)^s) \in O(s \cdot \log|\Sigma|)$, as desired. It remains to apply
  Lemma~\ref{lem:ourblumer}.
\end{proof}
We next comment on the fact that, in the $i$-th round of the SPELL
system, we try to fit ELQs that have at most $i-1$ existential
quantifiers, rather than ELQs of size at most $i$. By using the
following lemma and applying the same arguments as in the proof of
Theorem~\ref{thm:bfgen}, we obtain that our SAT-based approach yields
a PAC learning algorithm with sample size
$O\big (\frac{1}{\epsilon} \cdot \log \big (\frac{1}{\epsilon} \big )
\cdot \log \big (\frac{1}{\delta} \big ) \cdot |\Sigma| \cdot
    ||q_T|| \big )$. 
\begin{lemma}
  \label{lem:vcel} 
  Let \Omc be an \ELHdr-ontology, $n \geq 1$, and
  $\text{ELQ}^\exists_{(n)}$ the set of ELQs that have at most $n$
  existential restrictions. Then the VC-dimension of
  $\text{ELQ}^\exists_{(n)}$ w.r.t.\ \Omc is at most
  \mbox{$2(|\Sigma|+1)n$}. 
\end{lemma}
\noindent
\begin{proof}
  Let $n \geq 1$. We first observe
  that the number of concepts in ELQ$^\exists_{(n)}$ is bounded from above by
  $m_n= 4^{(|\Sigma|+1)n}$. To see this, note that the number of
  rooted, directed, unlabeled trees with $n$ nodes is bounded from
  above by the $n$-th Catalan number, which in turn is bounded from
  above by~$4^n$~\cite{DUTTON1986211}. Each such tree gives rise to an
  ELQ by assigning a unique role name from $\Sigma$ to each of
  the at most $n-1$ edges of the tree and a set of concept names from
  $\Sigma$ to each of the at most $n$ nodes of the tree. This clearly
  yields the stated bound $m_n$. Then trivially, the VC-dimension of
  ELQ$^\exists_{(n)}$ w.r.t.\ the empty ontology is at most $\log m_n$, thus
  $2(|\Sigma|+1)n$. Making the ontology non-empty may only decrease
  the VC-dimension as it may make non-equivalent concepts equivalent,
  but not vice versa.
  %
  %
%
\end{proof}
It is easy to see that (the proof of) Lemma~\ref{lem:vcel} applies
also to other DLs such as \ELI and \ALCI.

\section{Proof of Theorem~\ref{thm:notefficient}}

We split the three Points in Theorem~\ref{thm:notefficient} into three separate theorems.

\begin{theorem}
  Let $\Amf$ be a fitting algorithm for ELQs that always produces a most specific 
  fitting, if it exists. Then $\Amf$ is not a
  sample-efficient PAC learning algorithm.
\end{theorem}
\begin{proof}
  Assume to the contrary of what
  we aim to show that there is a sample-efficient PAC learning algorithm that
  produces a most specific fitting concept, if it exists, with
  polynomial function
  $m:\mathbb{R}^2 \times \mathbb{N}^4 \rightarrow \mathbb{N}$ as in
  Definition~\ref{def:efflearn}. Choose $\Sigma=\{A,r\}$, $q_T=A$,
  $\delta=\epsilon = 0.5$, and $n$ even and large enough such that
  $${n\choose {n/2}}> 2m(1/\epsilon,1/\delta,0,|\Sigma|,||q_T||,n(n+1)).$$
  We next construct positive examples; negative examples are not used.
  Let $\Smc$ denote the set of subsets of $\{1,\dots,n\}$ and let
  $\Smc^{\frac{1}{2}}$ be defined likewise, but include only sets of
  cardinality exactly $n/2$. With every $S \in \Smc$, we associate
  the database
  $$
     \Dmc_S = \{ r(b_0,b_1),\dots,r(b_{n-1},b_n) \} \cup \{ A(b_i)
     \mid i \in S \}
  $$
  as well as the pointed database $(\Dmc'_S,a_0)$ that can be obtained
  by starting with $\{ A(a_0) \}$ and then taking, for every
  $i \in S$, a disjoint copy of $\Dmc_{\{1,\dots,n\} \setminus \{i\}}$
  and identifying the root $b_0$ with~$a_0$.  Note that every
  $(\Dmc'_S,a_0)$ is a positive example for~$q_T$.

  \smallskip A crucial property of $(\Dmc'_S,a_0)$ is that it is a
  simulation dual of $(\Dmc_S,a_0)$ restricted to structures
  $(\Dmc_{S'},a_0)$, meaning the following.\footnote{For
    homomorphisms, the notion of a restricted duality is
    well-established,
    see for example  \cite{DBLP:books/daglib/0030491}.}

  \smallskip\noindent
  {\bf Claim.} For all $S,S' \in \Smc$:
$$
  (\Dmc_S, a_0) \preceq_\Sigma (\Dmc_{S'}, a_0) \text{ iff }
  (\Dmc_{S'}, a_0) \not\preceq_\Sigma (\Dmc'_S, a_0).
$$
The claim is easy to verify. We do not work with unrestricted
simulation duals here because we want the databases $(\Dmc'_S,a_0)$ to
be acyclic, and unrestricted simulation duals are not.

  Let $P$ be the probability distribution that assigns
  probability~$1/|\Smc^{\frac{1}{2}}|$ to every $(\Dmc'_S,a_0)$
  with $S \in \Smc^{\frac{1}{2}}$, and probability~$0$ to all other
  pointed databases.

  Now assume that the algorithm is started on a collection of
  $m(1/\epsilon,1/\delta,0,|\Sigma|,||q_T||,n(n+1))$ labeled data
  examples~$E$. Since all examples are acyclic, the most specific
  fitting $q_H$ exists \cite{DBLP:conf/aaai/JungLW20} and is output by
  the algorithm.  It is not important to make explicit at this point
  the exact details of $q_H$, but it can be thought of as the direct
  product of all the examples in $E$, viewed as an ELQ.

  To obtain a
  contradiction, it suffices to show that with probability at least
  $1 - \delta=0.5$, we have
  $\mn{error}_{P,q_T}(q_H) > \epsilon = 0.5$.
%
  We argue that, in fact, $q_H$ violates all data examples
  $(\Dmc'_S,a_0)$ with $S \in \Smc^{\frac{1}{2}}$ that are not in the
  sample $E$. The definition of $P$ and choice of $n$ then yield that
  with probability~1, $\mn{error}_{P,q_T}(q_H) > 0.5$.

  Thus take $S \in \Smc^{\frac{1}{2}}$ with $(\Dmc'_S,a_0) \notin E$.
  To show that $a_0 \notin q_H(\Dmc'_S)$, it suffices to prove the
  following:
  \begin{enumerate}

  \item $(\Dmc_{S},a_0) \preceq (\Dmc_{q_H},a_{q_H})$;

    Let  $q_S$ be
    $(\Dmc_{S},a_0)$ viewed as an ELQ.
    We show that $q_S$ is a fitting of $E$.  Take any
    $(\Dmc_{S'},a_0) \in E$. Then $S \neq S'$ and thus
    $(\Dmc_{S'},a_0) \not\preceq (\Dmc_{S},a_0)$. The claim yields
    $(\Dmc_{S},a_0) \preceq (\Dmc'_{S'},a_0)$. Thus $a_0 \in
    q_S(\Dmc'_{S'})$, and we are done.

    Since $q_H$ is the most specific fitting of $E$, it follows from
    $q_S$ being a fitting that $q_H \sqsubseteq q_S$, which yields
    $(\Dmc_{S},a_0) \preceq (\Dmc_{q_H},a_{q_H})$ as desired.
    
  \item $(\Dmc_{S},a_0) \not\preceq (\Dmc'_{S},a_0)$.

    Follows from the claim and the fact that
    $(\Dmc_{S},a_0) \preceq (\Dmc_{S},a_0)$.
    
  \end{enumerate}
  Now,  $a_0 \notin q_H(\Dmc'_S)$ follows from
  $(\Dmc_{q_H},a_{q_H})
  \preceq  (\Dmc'_{S},a_0)$ which is ruled out by Points~1 and~2
  above and the fact that the composition of two simulations is
  again a simulation.
\end{proof}

\begin{theorem}
  Let $\Amf$ be a fitting algorithm for ELQs that always produces a most general 
  fitting, if it exists. Then $\Amf$ is not a
  sample-efficient PAC learning algorithm.
\end{theorem}
\begin{proof}
  We only provide the missing details from the proof in the main part,
  that is, the proof of Points~1 and~2 stated in the main part and the
  place of the proof that say ``it follows from Point~1 that'':
  \begin{enumerate}

  \item $q$ is the most general ELQ that fits $(\Dmc'_q, a'_q,-)$;

      Let $p$ be an ELQ such that $(\Dmc'_q, a'_q)$ is a negative
      example for $p$. We have to show that $p \sqsubseteq q$.

      $(\Dmc'_q, a'_q)$ being a negative example for $p$ means that
      $a'_q \notin p(\Dmc'_q)$ and thus
      $(\Dmc_p,a_p) \not\preceq (\Dmc'_q,a'_q)$ by
      Lemma~\ref{lem:char-containment}. The definition of duals thus
      yields $(\Dmc_q,a_q) \preceq (\Dmc_p,a_p) $ and
      Lemma~\ref{lem:char-containment} gives $p \sqsubseteq q$,
      as desired.
      
    \item For all $T \subseteq S$, $q \notin T$ implies
      $p_T \not\sqsubseteq q$ where $p_T = \bigsqcap_{p \in T} p$.

      Consider the database $\Dmc_{p_T}$. Then clearly
      $a_{p_{T}} \in p_T(\Dmc_{p_T})$. But since $q \notin T$,
      $\Dmc_{p_T}$ contains no $r$-path outgoing from $a_{p_T}$ that
      contains the $A/B$-labeling of $q$, and thus
      $a_{p_{T}} \notin q(\Dmc_{p_T})$.
      
    \item It follows from Point~1 that
      $q_H = \bigsqcap_{(\Dmc'_q, a'_q) \in E} q$ is the most general
      ELQ that fits $E$.

      Clearly, $q_H \sqsubseteq q$ for every $(\Dmc'_q, a'_q) \in E$,
      and thus $(\Dmc_q,a_q) \preceq (\Dmc_{q_H},a_{q_H})$.  By Point~1,
      $(\Dmc_q,a_q) \not\preceq (\Dmc'_q,
      a'_q)$. Since the composition of two simulations is a
      simulation, this implies
      $(\Dmc_{q_H},a_{q_H}) \not\preceq (\Dmc'_q, a'_q)$. It follows
      that $q_H$ fits $E$.

      It remains to show that $q_H$ is most general. Assume that some
      ELQ $p$ fits $E$. Then $a'_q \notin p(\Dmc'_q)$ for all
      $(\Dmc'_q, a'_q) \in E$ and thus
      $(\Dmc_p,a_p) \not \preceq (\Dmc'_q, a'_q)$. By definition of
      duals, $(\Dmc_q,a_q) \preceq (\Dmc_p,a_p)$ and this is witnessed
      by some simulation $S_q$. But then $\bigcup_q S_q$ is a
      simulation showing $(\Dmc_{q_H},a_{q_H}) \preceq (\Dmc_p,a_p)$,
      and thus $p \sqsubseteq q_H$ as desired.
      
  \end{enumerate}
\end{proof}

\begin{theorem}\label{thm:minimaldepth}
  Let $\Amf$ be a fitting algorithm for ELQs that always produces a
  fitting of minimum quantifier depth. Then $\Amf$ is not a
  sample-efficient PAC learning algorithm.
\end{theorem}

\begin{proof}
  Assume to the contrary of what we aim to show that there is a
  sample-efficient learning algorithm that produces a most shallow
  fitting concept, if it exists, with associated polynomial function
  $m:\mathbb{R}^2 \times \mathbb{N}^4 \rightarrow \mathbb{N}$ as in
  Definition~\ref{def:efflearn}. We are going to use target queries of
  the form $q_T = \exists t^{n + 1}.\top$.

  Choose $\Sigma=\{r, s, t\}$,  $\delta = 0.5$, $\epsilon = 0.4$, and
  $n$ large enough such that 
  \begin{equation}
    \frac{2^n!}{2^{n p(n)} (2^n - p(n))!} > 1 - \delta
    \label{eq:choose-k-different}
    \tag{$\ast$}
  \end{equation}
  where $p(n)$ is the polynomial
  \[p(n)= m(\frac{1}{\delta}, \frac{1}{\epsilon}, 0, |\Sigma|, n + 1,
    p'(n))\] and $p'$ is a fixed polynomial that describes the size of
  the examples that we are going to use. Lemma~\ref{lem:limit} below
  shows that such an $n$ always exists, regardless of the precise
  polynomial~$p'$. The meaning of the expression on the left-hand side
  of~\eqref{eq:choose-k-different} will be explained later.

  Recall that the target query $q_T = \exists t^{n + 1}.\top$ is of
  quantifier depth $n+1$.  We construct (both positive and negative)
  examples such that with high probability, the drawn examples admit a
  fitting of quantifier depth $n$ that, however, does not generalize
  well.  Define a set of ELQs
  \[ S = \{ \exists r_1.\dots \exists r_n.\top \mid r_i \in \{r, s\},\
  1 \leq i \leq n\}. \]
  By
  Theorem~\ref{thm:duals}, each $q\in S$ has a polynomially sized
  $\Sigma$-dual that consists of a single element $(P_q,a)$.  By
  duality, $(P_q,a)$ is a positive example for $q_T$. Also by
  Theorem~3, $q_T$ has a polynomially sized $\Sigma$-dual that
  contains a single element $(\Dmc_T,a)$. For each $q\in S$, we
  construct a negative example $(N_q,a)$ by taking
  \[(N_q,a) = (P_q,a)\times (\Dmc_T,a)\]
  where $\times$ denotes the direct product of two pointed databases.
  Since the direct product is of polynomial size, both the positive
  examples and the negative examples are of size polynomial in $n$. We
  let $p'(n)$ be any polynomial that bounds (from above) the size of
  the examples. 


Note that by the properties of duals and products, for all $q \in S$
and for all $\Sigma$-ELQs $q'$, we have
\begin{enumerate}[label=(\roman*)]

  \item $a \in q'(P_q)$ iff $q' \not\sqsubseteq q$, and

  \item $a \in q'(N_q)$ iff $q' \not\sqsubseteq q$ and $q'
    \not\sqsubseteq q_T$.

\end{enumerate}
To see Point~(i) note that $a\in q'(P_q)$ iff (by
Lemma~\ref{lem:char-containment}) $(\Dmc_{q'},a_{q'})\preceq (P_q,a)$
iff (by duality) $(\Dmc_q,a_q)\not\preceq (\Dmc_{q'},a_{q'})$ iff (by
Lemma~\ref{lem:char-containment}) $q'\not\sqsubseteq q$.  Point~(ii)
can be shown similar and uses that $(\Dmc,a)\preceq (\Dmc_1,a_1)\times
(\Dmc_2,a_2)$ iff $(\Dmc,a)\preceq (\Dmc_1,a_1)$ and
$(\Dmc_a)\preceq(\Dmc_2,a_2)$, for all pointed databases
$(\Dmc,a),(\Dmc_1,a_1),(\Dmc_2,a_2)$.

Let $P$ be the probability distribution that assigns probability
$\frac{1}{2^{n + 1}}$ to every $(P_q, a)$ and $(N_q, a)$, and
probability $0$ to all other pointed databases.  Now, assume that the
algorithm is started on a collection of
$k=m(1/\delta,1/\epsilon,0,|\Sigma|,n+1,p'(n))$ pointed databases $E$
labeled according to $q_T$ and outputs a hypothesis $q_H$.

Note that the probability of sampling $\ell$ \emph{different}
objects from an $N$-element set is the ratio of those
sequences of length $\ell$ that contain pairwise distinct
elements in the set of all sequences of length $\ell$, that is,
\[\frac{\prod_{i=0}^{\ell-1}(N-\ell)}{N^\ell}=\frac{N!}{N^\ell\cdot
(N-\ell)!}.\]
We apply this observation to $N=2^n$ and $\ell=k$. 
By choice of $n$, with probability $>1 - \delta$ we have that for no 
$q \in S$, both $(N_q, a) \in E$ and $(P_q, a) \in E$.
To derive a contradiction, we show that the error of $q_H$ is strictly larger than $\epsilon$ if this is the case.

Consider the ELQ $q' = \bigsqcap_{(N_p, a,-) \in E} p$. We claim that
$q'$ fits $E$.
Note that $q\not\sqsubseteq q_T$ for any $q\in S$.
Point~(ii) then implies $a\notin q(N_q)$ for all $q\in S$ and thus $q'$ fits all negative
examples. Together with our assumption that for no
$q \in S$, both $(N_q, a) \in E$ and $(P_q, a) \in E$,
Point~(i) implies that $a\in p(P_q)$ for all $(N_p,a,-)\in E$ and
$(P_p,a,+)\in E$. 

Since $q'$ is a fitting of depth $n$ and the algorithm finds a fitting of minimal
depth, $q_H$ must have depth at most $n$, which implies that $q_H \not
\sqsubseteq q_T$.

Consider all $q \in S$. It must be that either $q_H \sqsubseteq q$ or
$q_H \not\sqsubseteq q$.
In the first case, Point~(i) implies $a \notin q_H(P_q)$, hence $q_H$ labels the
(positive) example
$(P_q, a)$ incorrectly. 
In the second case, Point~(ii) implies $a \in q_H(N_q)$, hence $q_H$
labels the (negative) example $(N_q, a)$ incorrectly.
Therefore, $\mn{error}_{P, q_T}(q_H) \geq 0.5 > \epsilon$.
\end{proof}

\begin{lemma}\label{lem:limit}
  For every polynomial $p(n)$, 
  \[\lim_{n\to\infty}(
  \frac{2^n!}{2^{np(n)} (2^n - p(n))!}) = 1. \]
\end{lemma}

\begin{proof}
  As argued in the proof of Theorem~\ref{thm:minimaldepth}, the term
  inside the limit is a probability, so the limit is at most 1.
  We start with bounding the expression inside the limit from below.
  \begin{align*}
    \frac{2^n!}{2^{np(n)} (2^n - p(n))!} & = \frac{2^n\cdot
    (2^n-1)\cdot\dots\cdot (2^n-p(n)+1)}{(2^n)^{p(n)}} \\
    &\geq
    \frac{(2^n-p(n)+1)^{p(n)}}{(2^n)^{p(n)}} \\
    &= 
    \left(1-\frac{p(n)+1}{2^n}\right)^{p(n)}
  \end{align*}
  It clearly suffices to show that the limit of the last expression is
  $1$. In order to do so, we reformulate the expression to avoid the $p(n)$ in
  the exponent. 
  \begin{align*}
    & \lim_{n\to\infty}(\left(1-\frac{p(n)+1}{2^n}\right)^{p(n)}) \\
    = &
    \lim_{n\to\infty}(\exp(\ln(\left(1-\frac{p(n)+1}{2^n}\right)^{p(n)})))
    \\
    = & 
    \exp(\lim_{n\to\infty}(\ln(\left(1-\frac{p(n)+1}{2^n}\right)^{p(n)})))
    \\
    = &
    \exp(\lim_{n\to\infty}(p(n)\cdot \ln(1-\frac{p(n)+1}{2^n}))).
  \end{align*}
  To determine the limit of a product where one factor
  $p(n)$ converges to $\infty$ and the other $\ln(\cdot)$ converges to
  $0$, we apply l'H{\^o}pital's
  rule. Set $f(n)=\ln(1-\frac{p(n)+1}{2^n})$ and
  $g(n)=1/p(n)$, so $\lim_{n\to\infty}\frac{f(n)}{g(n)}$ is exactly
  the limit we want to determine (inside the $\exp(\cdot)$). 
  L'H{\^o}pital's rule says that if
  $\lim_{n\to\infty}\frac{f'(n)}{g'(n)}$ exists, then 
  $\lim_{n\to\infty}\frac{f'(n)}{g'(n)}=\lim_{n\to\infty}\frac{f(n)}{g(n)}$.
  The derivations $f'(n)$ and $g'(n)$ of $f(n)$ and $g(n)$ are: 
  \begin{align*}
    f'(n) & = \frac{\ln(2)(p(n)+1)-p'(n)}{2^n-p(n)-1} \\
    g'(n) & = \frac{-p'(n)}{q(n)} \text{ for some polynomial $q(n)$}
  \end{align*}
  It remains to observe that $f'(n)/g'(n)$ is an expression that has
  an exponential $2^n$ in its numerator and everywhere else only
  polynomials. Thus, 
  $\lim_{n\to\infty}\frac{f'(n)}{g'(n)}=0=\lim_{n\to\infty}\frac{f(n)}{g(n)}$,
  which yields $\exp(0)=1$ as desired. 
\end{proof}

\section{Proof of Theorem~\ref{thm:duals} and Simulation Duals}\label{sec:duals}

Instead of proving Theorem~\ref{thm:duals}, we directly prove the more
general version in which databases $\Dmc_q$ for an ELQ $q$ are replaced with
DAG-shaped databases. A database $\Dmc$ is \emph{DAG-shaped} if the
directed graph $G_\Dmc=(\mn{adom}(\Dmc),\{(u,v)\mid r(u,v)\in
\Dmc\})$ is a DAG. 

Let \Dmc be a database and $a,b\in\mn{adom}(\Dmc)$. A \emph{path from
$a$ to $b$ in \Dmc} is a finite sequence
$a_1,r_1,a_2,\ldots,r_{k-1},a_k$ such that $a_1=a$, $a_k=b$, and
$r_i(a_i,a_{i+1})\in \Dmc$, for $1\leq i< k$. Note that role
assertions may be traveled forwards, but not backwards.  The
\emph{codepth} of an individual $a$ in a DAG-shaped database $\Dmc$ is
the length of the longest path starting in $a$; the
codepth of an individual $a$ such that there is no assertion
$r(a,b)\in \Dmc$ is defined to be $0$.

\begin{theorem}\label{thm:dag-duals} Let $\Sigma$ be a finite
signature and $(\Dmc,a)$ be a pointed database such that
$\Dmc$ is DAG-shaped. Then, we can compute
in polynomial time a $\Sigma$-simulation dual $M$ of $(\Dmc, a)$ such that
\begin{itemize}
  \item $||M|| \leq 3 \cdot |\Sigma| \cdot ||\Dmc||^3$, and
  \item if $\Dmc$ is tree-shaped, then $||M||\leq 3\cdot|\Sigma|\cdot ||\Dmc||^2$.
\end{itemize}

Moreover, if $\Dmc$ contains exactly one
$\Sigma$-assertion that mentions $a$, then the computed $M$ is actually a
singleton set.  
\end{theorem}

\begin{proof}
  Let $(\Dmc,a)$ be a pointed database with $\Dmc$ DAG-shaped and $\Sigma$ a
  finite signature. We construct a $\Sigma$-simulation dual of
  $(\Dmc,a)$ as follows. First, we define a database $\Dmc^*$ with
  domain 
  \begin{align*}
    \mn{adom}(\Dmc^*) ={} & \{b^\top\}\cup {}\\
    & \{\langle b,A(b)\rangle\mid
      A(b)\in \Dmc, A\in \Sigma\}\cup{} \\
    & \{\langle b,r(b,c)\rangle \mid
      r(b,c)\in \Dmc, r\in \Sigma\}
  \end{align*}
  and include the following assertions, for all $\langle b,A(b)\rangle
  \in \mn{adom}(\Dmc^*)$ and 
  $\langle b,r(b,c)\rangle \in \mn{adom}(\Dmc^*)$:
  \begin{enumerate}[label=(\roman*)]

    \item $B(b^\top)$ for all $B\in \Sigma\cap \NC$;

    \item $s(b^\top,b^\top)$ for all $s\in\Sigma\cap \NR$;

    \item $B(\langle b,A(b)\rangle)$ for all $B\in \Sigma\cap \NC$
      with $B\neq A$;

    \item $s(\langle b,A(b)\rangle, b^\top)$ for all $s\in \Sigma\cap \NR$;
    
    \item $B(\langle b,r(b,c)\rangle)$ for all $B \in \Sigma \cap \NC$;

    \item $ s(\langle b,r(b,c)\rangle, b^\top)$ for all
      $s\in\Sigma\cap \NR$ with $s\neq r$;
      
    \item $ r(\langle b,r(b,c)\rangle, \langle c,\alpha\rangle)$ for
      all $\langle c,\alpha\rangle\in\mn{adom}(\Dmc^*)$.
  \end{enumerate}

  We prove two auxiliary claims. 

  \noindent\textit{Claim 1.} For all $b \in
  \mn{adom}(\Dmc)$ and
  $\langle b,\alpha\rangle \in \mn{adom}(\Dmc^*)$, $(\Dmc, b) \not\preceq_\Sigma
  (\Dmc^*, \langle b,\alpha\rangle)$.

  \smallskip \noindent\textit{Proof of Claim~1.} We prove the claim by
  induction on the codepth of $b$ in $\Dmc$.  If
  $b$ has codepth $0$, then $\alpha$ is of the form $A(b)$, for
  $A(b)\in \Dmc$. By Point~(iii) in the
  definition of $\Dmc^*$, $A( \langle b, A(b)\rangle)\notin \Dmc^*$,
  and thus $(\Dmc,b)\not \preceq_\Sigma (\Dmc^*,\langle
  b,\alpha\rangle)$.

  Now, let $b$ have codepth greater than $0$. We distinguish cases on
  the shape of $\alpha$. 
  \begin{itemize}

    \item If $\alpha$ is of the form $A(b)$ for some $A(b)\in
      \Dmc$, then we can argue as in the base
      case that $(\Dmc,b)\not \preceq_\Sigma (\Dmc^*,\langle
      b,\alpha\rangle)$.

    \item If $\alpha$ is of the form $r(b,c)$ for some $r(b,c)\in
      \Dmc$, assume for contradiction that there
      is a $\Sigma$-simulation $S$ from $\Dmc$ to $\Dmc^*$ with $(b,
      \langle b,r(b,c)\rangle \in S)$. Since $S$ is a simulation and
      $c$ is an $r$-successor of $b$ in \Dmc, there has to be an
      $r$-successor $c'$ of $\langle b,r(b,c)\rangle$ in $\Dmc^*$ with
      $(c,c')\in S$. By Point~(vi) and~(vii), 
      $c'$ is of shape $\langle c,\alpha\rangle$. But then
      $(\Dmc,c)\preceq_\Sigma(\Dmc',\langle c,\alpha\rangle)$,
      contradicting the induction hypothesis.

  \end{itemize}

  \smallskip
  \noindent\textit{Claim 2.} For all $b \in \mn{adom}(\Dmc)$ and pointed
  databases $(\Dmc', c)$, if $(\Dmc, b) \not
  \preceq_\Sigma (\Dmc', c)$ then there is a $\langle b,\alpha\rangle
  \in \mn{adom}(\Dmc^*)$
  such that $(\Dmc', c) \preceq_\Sigma (\Dmc^*, \langle
  b,\alpha\rangle)$.

  \smallskip \noindent\textit{Proof of Claim~2.} We prove the claim by
  induction on the codepth of $b$ in $\Dmc$.  If
  $b$ has codepth $0$ and $(\Dmc, b) \not \preceq_\Sigma (\Dmc', c)$,
  then there is a concept name $A \in \Sigma$ such that $A(b) \in
  \Dmc$ and $A(c) \notin \Dmc'$. It can be verified using
  Points~(i)--(iii) above that the relation
  \[ S = \{ (c, \langle b,A(b)\rangle) \} \cup \{ (c', b^\top) \mid c' \in \mn{adom}(\Dmc')
  \} \]
  is a $\Sigma$-simulation from $\Dmc'$ to $\Dmc^*$ with
  $(c, \langle b,A(b)\rangle) \in S$ as required.
    
  Now, let $b$ have codepth greater than $0$ and assume 
  $(\Dmc, b)
  \not\preceq_\Sigma (\Dmc', c)$. We distinguish cases on why the
  latter is the case: 
  \begin{itemize}

    \item If there is a concept name $A\in \Sigma$ such that $A(b)\in
      \Dmc$ and $A(c)\notin \Dmc'$, we can argue as in the base case
      that $(\Dmc',c)\preceq_\Sigma (\Dmc^*,\langle b,A(b) \rangle)$. 

    \item If there is an assertion $r(b,
      b') \in \Dmc$ such that for all $r(c, c') \in \Dmc'$, $(\Dmc,
      b') \not\preceq_\Sigma (\Dmc', c')$. We show that $(\Dmc', c)
      \preceq_\Sigma (\Dmc^*, \langle b,r(b, b')\rangle )$.

      The induction hypothesis implies that for all $r(c, c') \in
      \Dmc'$ there is an $\langle b',\beta\rangle \in
      \mn{adom}(\Dmc^*)$ and a simulation $S_{c'}$ from $\Dmc'$ to
      $\Dmc^*$ with $(c', \langle b',\beta\rangle) \in S_{c'}$. It can
      be verified using Points~(v)-(vii) above that 
      \begin{align*}
	S ={} & \{ (b, \langle b,r(b,b')\rangle) \} \cup \{ (c',
	  b^\top) \mid c' \in \mn{adom}(\Dmc') \} \cup{} \\ &
	  \bigcup_{r(c, c') \in \Dmc'} S_{c'} \end{align*}
	is a simulation from $\Dmc'$ to $\Dmc^*$ with $(b,\langle
	b,r(b, b')\rangle)\in S$.

    \end{itemize}

    This completes the proofs of Claims~1 and~2. The next claim shows
    how to read off a simulation dual of $(\Dmc,a)$ from~$\Dmc^*$.

  \smallskip\noindent\textit{Claim 3.} The set
  \begin{align*}
    M_a ={} & \{(\Dmc^*,\langle a,\alpha\rangle)\mid \langle a,\alpha
      \rangle \in\mn{adom}(\Dmc^*)\}
  \end{align*}
  is a $\Sigma$-simulation dual of $(\Dmc,a)$.

  \smallskip\noindent\textit{Proof of Claim 3.} Suppose
  $(\Dmc,a)\not\preceq_\Sigma (\Dmc',a')$ for some $(\Dmc',a')$. Then Claim~2
  implies that there is some $\langle a,\alpha\rangle\in
  \mn{adom}(\Dmc^*)$ with $(\Dmc',a')\preceq (\Dmc^*,\langle
  a,\alpha\rangle)$. It remains to note that $(\Dmc^*,\langle
  a,\alpha\rangle)\in M_a$. Conversely, suppose that
  $(\Dmc,a)\preceq_\Sigma (\Dmc',a')$ and assume for showing a
  contradiction that $(\Dmc',a')\preceq_\Sigma (\Dmc^*,\langle
  a,\alpha\rangle)$ for some $\langle a,\alpha\rangle
  \in\mn{adom}(\Dmc^*)$. Since $\preceq_\Sigma$ is transitive, we
  obtain $(\Dmc,a)\preceq_\Sigma (\Dmc^*,\langle
  a,\alpha\rangle)$, in contradiction to Claim~1. 
  This finishes the proof of Claim~3. 

  \smallskip Clearly, $M_a$ is a
  singleton set if $\Dmc$ contains only a single
  $\Sigma$-assertion mentioning $a$. It remains to analyze $||M_a||$. We start
  with analyzing $||\Dmc^*||$. Points~(i) and~(ii) together contribute 
  $|\Sigma|$ assertions. Points~(iii) and~(iv) contribute together
  $|\Sigma|\cdot n_C$ assertions where $n_C$ denotes the number of
  assertions of shape $A(b)$ in $\Dmc$.
  Points~(v) and~(vi) contribute $|\Sigma|\cdot n_R$ assertions 
  where $n_R$ denotes the number of
  assertions of shape $r(b,c)$ in $\Dmc$.
  Finally, Point~(vii) contributes $|\Dmc|^2$
  assertions. Overall, we obtain
  \[||\Dmc^*||\leq |\Sigma|+|\Sigma|\cdot n_C +|\Sigma|\cdot n_R +
  |\Dmc|^2\leq 3\cdot |\Sigma| \cdot |\Dmc|^2.\]
  Thus, $||M_a||\leq |\Dmc|\cdot 3 \cdot |\Sigma| \cdot |\Dmc|^2\leq 3 \cdot |\Sigma| \cdot ||\Dmc||^3$ as required. 

  If $\Dmc$ is tree-shaped, then the bound on the number of assertions that
  Point~(vii) contributes can be improved. Only a single incoming assertion is
  added for each $\langle c, \alpha\rangle$, resulting in $|\Dmc|$ assertions.
  This improves the overall bounds to
  \[||\Dmc^*||\leq |\Sigma|+|\Sigma|\cdot n_C +|\Sigma|\cdot n_R +
  |\Dmc| \leq 3\cdot |\Sigma| \cdot |\Dmc|.\]
  Thus, $||M_a||\leq |\Dmc|\cdot 3 \cdot |\Sigma| \cdot |\Dmc|\leq 3 \cdot |\Sigma| \cdot ||\Dmc||^2$ as required. 
  
\end{proof}

We next characterize the pointed databases that admit finite
simulation duals. We need some additional notation.

We say that $b$ is \emph{reachable} from $a$ if there is a path from
$a$ to $b$.  We use $\Dmc^{\downarrow a}$ to denote the database that
consists of all facts $A(b),r(b,b')\in \Dmc$ with $b,b'$ reachable
from~$a$.  We note that, for some $a\in \mn{adom}(\Dmc)$,
$\Dmc^{\downarrow a}$ might be empty (namely, if there are no
assertions of the form $A(a),r(a,b)\in \Dmc$). In a slight abuse of
notation, we then allow to write $(\emptyset,a)$ and mean the pointed
database $(\{\top(a)\},a)$.  We use $\Dmc_\Sigma$ to denote the
restriction of a database \Dmc to its $\Sigma$-assertions, for any
signature $\Sigma$. Hence, $\Dmc_\Sigma^{\downarrow a}$ is
the restriction of $\Dmc^{\downarrow a}$ to~$\Sigma$.

The proof of the
characterization relies on the (standard) notion of unravelings.  Let
\Dmc be a database and $a\in \mn{adom}(\Dmc)$. The \emph{unraveling of
$\Dmc$ at $a$} is the (possibly infinite) database $\Umc$ whose domain
$\mn{adom}(\Umc)$ consists of all paths starting in $a$ and that
contains the following assertions for every
$p=a_1,r_1,a_2,\ldots,r_{k-1},a_k\in \mn{adom}(\Umc)$:

\begin{itemize}

  \item $A(p)$ for all $A(a_{k})\in \Dmc$ and 

  \item $r_{k-1}(p',p)$ for $p' = a_1, r_1, \ldots, a_{k - 1}$.

\end{itemize}
\begin{restatable}{theorem}{thmcharduals} \label{thm:charduals}
  Let $\Sigma$ be a finite signature and $(\Dmc,a)$ a pointed
  database. Then, $(\Dmc,a)$ has a finite $\Sigma$-simulation dual iff
  $\Dmc_\Sigma^{\downarrow a}$ is DAG-shaped. 
%
\end{restatable}

\begin{proof}
  For the ``if''-direction, suppose that $\Dmc^{\downarrow a}_\Sigma$
  is DAG-shaped. It should be clear that both
  $(\Dmc,a)\preceq_\Sigma (\Dmc^{\downarrow a}_\Sigma,a)$ and
  $(\Dmc^{\downarrow a}_\Sigma,a)\preceq_\Sigma (\Dmc,a)$, and thus $(\Dmc,a)$ and
  $(\Dmc^{\downarrow a}_\Sigma,a)$ have the same $\Sigma$-simulation
  duals. Theorem~\ref{thm:dag-duals} implies the existence of a
  finite $\Sigma$-simulation for $(\Dmc^{\downarrow a}_\Sigma,a)$ and
  thus of $(\Dmc,a)$.
%

  For ``only if'', we assume that $\Dmc_\Sigma^{\downarrow a}$ is not
  DAG-shaped and show that there cannot be a finite $\Sigma$-simulation
  dual. Assume to the contrary of what is to be shown that $M$ is a
  finite $\Sigma$-simulation dual of $(\Dmc,a)$. Let $\Umc$ be the
  unraveling of $\Dmc_\Sigma^{\downarrow a}$ at $a$. Note that $\Umc$ is
  an infinite (and tree-shaped) database as
  $\Dmc_\Sigma^{\downarrow a}$ is not DAG-shaped. Let, moreover, $\Umc_i$
  denote the restriction of \Umc to individuals that have distance at
  most $i$ from the root $a$, for $i\geq 0$. Clearly, we have:
  \begin{enumerate}[label=(\roman*)]

    \item $(\Dmc,a)\preceq_\Sigma (\Umc,a)$, and

    \item $(\Dmc,a)\not\preceq_\Sigma (\Umc_i,a)$, for all $i\geq 0$.

  \end{enumerate}
  By duality and Point~(ii), for every $i\geq 0$  there exists some
  $(\Dmc',a')\in M$ with $(\Umc_i,a)\preceq_\Sigma (\Dmc',a')$.  Since
  $M$ is finite, there is some $(\Dmc^*,a^*)\in M$ such that
  $(\Umc_i,a)\preceq_\Sigma (\Dmc^*,a^*)$ for infinitely many $i\geq
  0$. Using a standard ``simulation skipping'' argument, we can
  inductively construct a simulation $S$ witnessing
  $(\Umc,a)\preceq_\Sigma (\Dmc^*,a^*)$. By duality, we obtain
  $(\Dmc,a)\not\preceq_\Sigma (\Umc,a)$, which is in contradiction to
  Point~(i) above.

  Let us now give some details regarding the simulation skipping
  argument. Let $I$ be an infinite set such that $(\Umc_i,a)\preceq_\Sigma 
  (\Dmc^*,a^*)$ for all $i\in I$, and let $(S_i)_{i\in I}$ be a family
  of
  $\Sigma$-simulations witnessing that. 
  We provide an infinite family
  $(S_i^*)_{i\geq 0}$ of relations $\mn{adom}(\Umc_i)\times
  \mn{adom}(\Dmc^*)$ such that, for every $i\geq 0$:
  \begin{enumerate}[label=(\alph*)]

    \item $S_i^*$ is a $\Sigma$-simulation witnessing
      $(\Umc_i,a)\preceq_\Sigma (\Dmc^*,a^*)$, and

    \item $S_i^*\subseteq S_j$, for infinitely many $j\in I$.


  \end{enumerate}
  We start with setting $S_0^*=\{(a,a)\}$ which clearly satisfies
  Points~(a) and~(b). To obtain $S_{i+1}^*$ from $S_i^*$, let
  $B=\mn{adom}(\Umc_{i+1})\setminus \mn{adom}(\Umc_i)$. Note that $B$
  is finite. By Point~(b) applied to $S_i^*$, there is an infinite set
  $J\subseteq I$ such that $S_i^*\subseteq S_j$ for every $j\in J$.
  Since both $B$ and $\Dmc^*$ are finite, we can pick an infinite
  subset $J'\subseteq J$ such that for every $b\in B$, every
  $d\in \mn{adom}(\Dmc^*)$, and every $j,j'\in J'$, we have
  \[(b,d)\in S_j\quad\text{ iff }\quad (b,d)\in S_{j'}.\]
  Obtain $S_{i+1}^*$ from $S_i^*$ by adding $(b,d)\in B\times \mn{adom}(\Dmc^*)$ in
  case $(b,d)\in S_j$ for all $j\in J'$. It is routine to verify that
  $S_{i+1}^*$ satisfies Points~(a) and~(b), and that 
  \[S = \bigcup_{i\geq 0}S_i^*\]
  witnesses $(\Umc,a)\preceq_\Sigma (\Dmc^*,a^*)$, as desired. 
  %
\end{proof}

\section{Proof of Theorem~\ref{thm:expdepthbound}}

Let $(\Lmc,\Qmc)$ be an OMQ language, \Omc an \Lmc-ontology, and
$\Sigma$ a finite signature. A
\emph{(downward) frontier} for a query $q \in Q$ with respect to
$\Omc$ and $\Sigma$ is a finite set $F \subseteq \Qmc$ such that
\begin{enumerate}

  \item each $p \in F$ is a downward refinement of $q$ w.r.t.~$\Omc$ and

  \item for each $p \in \Qmc_\Sigma$ that is a downward refinement of
    $q$ w.r.t.~$\Omc$, there is some $p'\in F$ such that $\Omc\models
    p\qcontains p'$.

\end{enumerate}
Note that for both refinement operators $\rho_1$ and $\rho_2$ defined in
the main part of the paper and any ELQ $q$ and signature $\Sigma$,
$\rho_i(q,\Sigma)$ is a downward frontier for $q$ with respect to the
empty ontology and $\Sigma$.
%
%

The following clearly implies Theorem~\ref{thm:expdepthbound}.
\begin{theorem}
Let $(\Lmc,\Qmc)$ be an OMQ language. The following are equivalent:
\begin{enumerate}
\item $(\Lmc,\Qmc)$ has an ideal downward 
  refinement operator,
\item $(\Lmc,\Qmc)$ has an ideal downward 
  refinement operator that is $2^{O(n)}$-depth bounded,
\item for all $\Lmc$-ontologies $\Omc$ and all finite signatures
  $\Sigma$,
  each $q \in \Qmc$ has a downward 
  frontier w.r.t.~$\Omc$ and $\Sigma$.
\end{enumerate}
\end{theorem}

\begin{proof}
%
\emph{From 1 to 3}.
Let $\rho$ be a downward refinement operator. We claim that, 
for each $\Lmc$-ontology $\Omc$ and $q \in \Qmc$ and finite signature
$\Sigma$, 
$\rho(q,\Omc,\Sigma)$ is a downward frontier of $q$ w.r.t.~\Omc and
$\Sigma$. For Point~1 in the definition of downward frontier, note
that any $p\in \rho(q,\Omc,\Sigma)$ is a downward refinement of
$q$ w.r.t.~\Omc. For Point~2, let $p\in \Qmc_\Sigma$ be any downward refinement of $q$
w.r.t.~$\Omc$.
By completeness, there is a finite sequence
$q_1, \ldots, q_n$  with $q_1=q$, $q_n=p$,
and $q_{i+1}=\rho(q_i,\Omc,\Sigma)$ for all $i$.  Note that, necessarily, $n\geq 2$. It follows that 
$q_2\in \rho(q,\Omc,\Sigma)$ and $\Omc\models p\sqsubseteq q_2$.

\smallskip
\emph{From 3 to 2}.
Take any $\Lmc$-ontology $\Omc$ and finite signature $\Sigma$.
For each $q \in \Qmc$, let
$F(q,\Omc,\Sigma)$ be a downward frontier for $q$ w.r.t.~$\Omc$ and
$\Sigma$.
Let $\rho(q,\Omc,\Sigma)$ be the union of $F(q,\Omc,\Sigma)$ 
with the set of all downwards refinements $q'\in Q_\Sigma$ of $q$ with
$||q'||\leq ||q||$.
Clearly, $\rho$ is a finite downward refinement operator.
To show that $\rho$ is complete, consider any pair of queries $(q,p)$
from \Qmc such that
$\Omc\models p\sqsubseteq q$. Suppose for the sake of a 
contradiction that there is no downward
$\rho,\Omc,\Sigma$-refinement sequence starting in $q$ and ending in a query
$p'$ with $\Omc \models p \equiv p'$. 
It then follows from the properties of downward frontiers
that there exists an infinite sequence of (pairwise
non-equivalent) queries
\[ q_1, q_2, \ldots \]
with $q_1=q$ and $q_{i+1}\in F(q_i,\Omc,\Sigma)$ for all $i\geq 0$, such that
$p$ is a downward refinement of each $q_i$ w.r.t. $\Omc$.
Let $k>0$ be minimal with
$||q_k|| \geq ||p||$. Clearly, $k=2^{O(||p||)}$.
Moreover, $q_1, \ldots, q_k, p$ is a $\rho,\Omc,\Sigma$-refinement
sequence starting in $q$ and ending in $p$, a contradiction.
Hence, $\rho$ is an ideal downward refinement operator. Furthermore,
$\rho$ is $2^{O(n)}$-depth bounded.

\smallskip
The implication from 2 to 1 is trivial.
\end{proof}

\section{Proof of Lemma~\ref{lem:smallneighbor} and Theorem~\ref{thm:aonepac}}

Before proving Lemma~\ref{lem:smallneighbor} and
Theorem~\ref{thm:aonepac}, we recall some important properties of
minimal ELQs. Recall that an ELQ $q$ is \emph{minimal} if there is no
ELQ $p$ with $p\equiv q$ and $\numat{p}<\numat{q}$. Due to the correspondence of ELQs and \EL-concepts, we may speak of
minimal \EL-concepts. Minimal ELQs (and thus, minimal \EL-concepts) can
be characterized in terms of functional simulations, where a
simulation $S$ between $\Dmc_1$ and $\Dmc_2$ is called
\emph{functional} if for every $d\in \mn{adom}(\Dmc_1)$, there is at
most one $e\in \mn{adom}(\Dmc_2)$ with $(d,e)\in S$. 
\begin{lemma}\label{lem:min-elqs}
  An ELQ $q$ is minimal iff the only functional simulation $S$ from
  $\Dmc_q$ to $\Dmc_q$ with $(a_q,a_q)\in S$ is the identity.
\end{lemma}

The following lemma shows several ways how to refine a minimal
\EL-concept. Its proof is straightforward using simulations and
Lemmas~\ref{lem:char-containment} and~\ref{lem:min-elqs}, details are
left to the reader.
\begin{lemma} \label{lem:strengthen-minimal}
  Let $C$ be a minimal \EL-concept and $D=A_1\sqcap \cdots \sqcap
  A_k\sqcap \exists r_1.D_1\sqcap \cdots \sqcap \exists r_\ell.D_\ell$
  a subconcept of $C$. Then the following hold: 
  \begin{enumerate}

    \item For all $C'$ that can be obtained from $C$ by replacing 
      $D$ with $D\sqcap A$ for some concept name
      $A\notin\{A_1,\ldots,A_k\}$, we have
      $C'\sqsubseteq C$ and $C\not\sqsubseteq C'$.

    \item For all $C'$ that can be obtained from $C$ by replacing $D$
      with $D\sqcap \exists r.\top$ for some role name
      $r\notin\{r_1,\ldots,r_\ell\}$, we have $C'\sqsubseteq C$ and
      $C\not\sqsubseteq C'$.

    \item For all $C'$ that can be obtained from $C$ by replacing a
      $D$ with $D\sqcap \exists r.\widehat D$ for some role name
      $r\in\{r_1,\ldots,r_\ell\}$ and a concept $\widehat D$ such that $D_j\not\sqsubseteq \widehat D$ for
      all $j$ with $r=r_j$, we have $C'\sqsubseteq C$ and
      $C\not\sqsubseteq C'$.

  \end{enumerate}
\end{lemma}

The following is a slight strengthening of~\cite[Lemma~6 in the full
paper]{DBLP:conf/aaai/JungLW20}. Recall that $\Dmc^{\downarrow a}$
denotes the set of all assertions $A(b),r(b,b')$ in \Dmc such that
$b,b'$ are reachable from $a$, c.f.\ Section~\ref{sec:duals}.
\begin{lemma}\label{lem:sizebound}
  Let $(\Dmc_1,a_1)$ and $(\Dmc_2,a_2)$ be pointed databases with
  $\Dmc_2$ tree-shaped. If $(\Dmc_1,a_1)\not\preceq(\Dmc_2,a_2)$, then
  there exists a set $\Dmc\subseteq \Dmc_1$ with $a_1\in
  \mn{adom}(\Dmc_1)$ such that $|\Dmc|\leq |\Dmc_{2}^{\downarrow
  a_2}|+1$ and $(\Dmc,a_1)\not\preceq (\Dmc_2,a_2)$.
\end{lemma} \begin{proof}\ The proof is by induction on the depth of
  $\Dmc^{\downarrow a_2}_2$. 

Assume first that $\Dmc_2^{\downarrow a_2}$ has depth $0$. If there exists a
concept name $A\in \NC$ with $A(a_1)\in \Dmc_1$ but $A(a_2)\not\in
\Dmc_2$, then $\Dmc=\{A(a_1)\}$ is as required.  Otherwise there
exists a role name $r\in \NR$ and $a'$ with $r(a_1,a')\in \Dmc_1$.
Then, $\Dmc=\{r(a_1,a')\}$ is as required. 

Now, suppose that $\Dmc_2^{\downarrow a_2}$ has depth $k>0$ and assume
$(\Dmc_1,a_1)\not\preceq (\Dmc_2,a_2)$. If there exists a
concept name $A\in \NC$ with $A(a_1)\in \Dmc_1$ but $A(a_2)\not\in
\Dmc_2$, then $\Dmc=\{A(a_1)\}$ is as required.  
Otherwise there exists a role name $r\in \NR$ and some $r(a_1,a')\in
\Dmc_1$ such that for all $b$ with $(a_2,b)\in \Dmc_2$, we have
$(\Dmc_1,a')\not\preceq (\Dmc_2,b)$.
Fix $a'$. By induction hypothesis, we can fix for every $b$ with
$r(a_2,b)\in \Dmc_2$, a subset $\Dmc_{b}\subseteq \Dmc_1$ with $a'\in
\mn{adom}(\Dmc_{b})$ such that $|\Dmc_{b}|\leq |\Dmc_2^{\downarrow b}|+1$ and
$(\Dmc_{b},a')\not\preceq (\Dmc_2,b)$. Let $\Dmc$ be the
union of $\{r(a,a')\}$ and all $\Dmc_{b}$ with $r(a_2,b)\in \Dmc_2$. Then
$\Dmc$ is as required.
\end{proof}

\lemsmallneighbor*
\begin{proof}
  
  Let $p,q$ be ELQs such that $p$ is a minimal downward neighbor of $q$, that
  is, $p\sqsubseteq q$, $q\not\sqsubseteq p$, and for all $p'$ with
  $p\sqsubseteq p'\sqsubseteq q$, we have $p'\equiv p$ or
  $q\equiv p'$. Since $\numat{q} \geq \numat{q'}$ for every minimal ELQ
  $q'$ with $q'\equiv q$,
  we may assume that also $q$ is minimal. 

  Let $(D_p,a_p)$ and $(D_q,a_q)$ be the pointed databases associated
  with $p$ and $q$, respectively. By Lemma~\ref{lem:char-containment}, there is a
  simulation $S$ from $\Dmc_q$ to $\Dmc_p$ with $(a_q,a_p) \in S$. We
  can w.l.o.g.\ assume $S$ to be functional.  Clearly, the inverse
  $S^-$ of $S$ is not a simulation from $\Dmc_p$ to $\Dmc_q$ since
  $q\not\sqsubseteq p$.  We distinguish two cases.

  \smallskip \textit{Case~1.} There is $(a,a')\in S$ such that
  $A(a')\in \Dmc_p$, but $A(a)\notin \Dmc_q$. Obtain $\Dmc_{q'}$ from
  $\Dmc_q$ by adding $A(a)$, and let $q'$ be the corresponding
  ELQ. Clearly, $S$ is a simulation from $\Dmc_{q'}$ to $\Dmc_p$,
  hence $p\sqsubseteq q'$.  Moreover, by construction and Point~1 of
  Lemma~\ref{lem:strengthen-minimal}, we have $q'\sqsubseteq q$ and
  $q\not\sqsubseteq q'$. Since $p$ is a downward neighbor of $q$ and
  $p \sqsubseteq q' \sqsubseteq q$, we thus have $p \equiv q'$.  Since
  $q'$ is obtained from $q$ by adding a single atom and
  $\numat{p} \leq \numat{q'}$, we obtain $\numat{p}\leq 2\numat{q}+1$ as required.

  \smallskip \textit{Case~2.} Case~1 does not apply and there is
  $(a,a')\in S$ and an assertion $r(a',b')\in \Dmc_p$
  such that there is no $b$ with $(b,b')\in S$. Choose such an $(a,a')$ such that
  $a'$ has maximal distance from the root~$a_p$. We distinguish
  two subcases.

  \begin{enumerate}[label=(\alph*)]

    \item If $a$ does not have an $r$-successor in $\Dmc_q$,
      obtain $\Dmc_{q'}$ from $\Dmc_q$ by adding an atom $r(a,b)$,
      for some fresh $b$. 

      Clearly, $S'=S\cup \{(b,b')\}$ is a functional simulation from
      $\Dmc_{q'}$ to $\Dmc_p$ with $(a_{q},a_p)\in S'$, hence $p\sqsubseteq q'$. Moreover, by
      construction and Point~2 of Lemma~\ref{lem:strengthen-minimal},
      $q'\sqsubseteq q$ and $q\not\sqsubseteq q'$.  Since $p$ is a
      downward neighbor of $q$, we have $q'\equiv p$.  Since $p$ is
      obtained from $q$ by adding a single atom and
      $\numat{p} \leq \numat{q'}$, we obtain $\numat{p}\leq 2\numat{q}+1$ as required.

    \item Otherwise, $a$ has $r$-successors $a_1,\ldots,a_k$ in
      $\Dmc_q$. Let $b_1,\ldots,b_k$ be the (uniquely defined)
      elements with $(a_i,b_i)\in S$ for every~$i$. In particular,
      $b'\notin\{b_1,\ldots,b_k\}$. Note that $(\Dmc_p,b')\not\preceq
      (\Dmc_q,a_i)$ for every $i\in\{1,\ldots,k\}$: otherwise, we
      would have $(\Dmc_p,b')\preceq (\Dmc_p,b_i)$ for some $i$ in
      contradiction to minimality of $p$.
%



      Let $\widehat\Dmc$
      be a minimal subset of $\Dmc_p$ such that $b'\in
      \mn{adom}(\widehat\Dmc)$ and $(\widehat
      \Dmc,b')\not\preceq(\Dmc_q,a_i)$ for
      all $i$. By Lemma~\ref{lem:sizebound}, 
      $|\widehat\Dmc|\leq (n_1+1)+\cdots+(n_k+1)$ where
      $n_i$ is the number of assertions in the tree rooted at
      $a_i$.
      It follows that $|\widehat\Dmc|\leq |\Dmc_q|$. 

      Now, obtain $\Dmc_{q'}$ from $\Dmc_q$ by adding
      $\widehat\Dmc$ (assuming that the individuals in
      $\Dmc_q$ and $\widehat\Dmc$ are disjoint) as well as the
      assertion $r(a,b')$. Note that $\numat{q'}=|\Dmc_{q'}|\leq |\Dmc_q|+|\widehat
      \Dmc|+1\leq 2|\Dmc_q|+1=2\numat{q}+1$. 

      Clearly, $S$ can be extended to a functional simulation
      from $\Dmc_{q'}$ to $\Dmc_p$, hence $p\sqsubseteq
      q'$. Moreover, by construction and Point~3 of
      Lemma~\ref{lem:strengthen-minimal}, $q'\sqsubseteq q$ and
      $q\not\sqsubseteq q'$.  Since $p$ is a downward neighbor of $q$,
      this implies $q'\equiv p$. Hence, $\numat{p}\leq
      \numat{q'}\leq 2\numat{q}+1$ as required. 

  \end{enumerate}
  This finishes the proof of the lemma.
\end{proof}
Let $\Sigma$ be a finite signature and $p,q\in\text{ELQ}_\Sigma$. A
\emph{$\Sigma$-specialization sequence from $p$ to $q$} is
a sequence $q_1,\ldots,q_k$ of queries from $\text{ELQ}_\Sigma$ such that $q_1=p$, $q_k=q$, and $q_{i+1}$
is a neighbor of $q_i$, that is, $q_{i+1}\in\rho_2(q_i,\Sigma)$, for
$1\leq i<k$.\footnote{Note the difference
  to refinement sequences where refinements are used in place of
  neighbors.} We recall two useful properties of the \EL-subsumption
lattice~\cite[Corollary~5.2.3]{KriegelPhD}, namely that for all ELQs
$p,q\in\text{ELQ}_\Sigma$ with $p\sqsubseteq q$,
\begin{enumerate}
\item[(I)] there is a 
(finite) $\Sigma$-specialization sequence from $q$ to $p$ and 

\item[(II)] all $\Sigma$-specialization sequences
from $q$ to $p$ have the same length (\emph{Jordan-Dedekind chain condition}).
\end{enumerate}
  
\lemrhosideal*
\begin{proof} Both $\rho_1$ and $\rho_2$ are finite by definition.  It
  follows from Property~(I) above and the definition of $\rho_2$, and
  Lemma~\ref{lem:smallneighbor} that $\rho_2$ is complete.  The same
  is then true for $\rho_1$ as it contains $\rho_2$ in the sense that
  $\rho_1(q,\Sigma) \subseteq \rho_2(q,\Sigma)$ for every ELQ $q$ and
  finite signature $\Sigma$.
\end{proof}
%
%
%
%
%
%
%
%
%
\thmaonepac*
\begin{proof}
  We start with analyzing $\Amf_2$. Let $E$ be the input to $\Amf_2$ and let
  $\Sigma=\mn{sig}(E)$.  By Theorem~\ref{thm:notefficient}, it
  suffices to show that $\Amf_2$ produces a most general fitting of $E$
  (if it exists).  By the Jordan-Dedekind chain condition, we can
  assign a \emph{level} to every ELQ $q$ defined as the length of
  $\Sigma$-specialization sequences from $\top$ to $q$. Let
  $M_1,M_2,\dots$ be the sequence of sets $M$ constructed by the
  breadth-first search algorithm $\Amf_2$, that is, $M_1= \{ \top \}$ and
  $M_{i+1}$ is obtained from $M_i$ by applying the refinement operator
  $\rho_2$. It is easy to show that for all $i \geq 0$, the set $M_i$
  contains precisely the ELQs of level $i$ that fit the positive
  examples $E^+$ in $E$.

  So suppose that a most general fitting $q^*$ exists and that $\Amf_2$
  returns some ELQ $q$ after $n$ rounds. Then $q \in M_n$. Since $q$
  is a fitting and $q^*$ is a most general fitting, we have
  $q\sqsubseteq q^*$.  By Property~(I) above, there is a non-empty
  $\Sigma$-specialization sequence from $q^*$ to $q$.  Thus, the level of $q^*$
  is strictly smaller than that of $q$. But then there is an $m<n$
  such that the set $M_m$ contains $q^*$, in contradiction to $q$
  being output by $\Amf_2$.

  \medskip To show that $\Amf_1$ is a sample-efficient PAC learning
  algorithm we show that it is an Occam algorithm.  Let
  $M_1,M_2,\dots$ be the sequence of sets $M$ constructed by the
  breadth-first search algorithm $\Amf_1$.  We show below
  that
  \begin{enumerate}[label=(\roman*)]

    \item for all $i \geq 1$, $q\in M_i$ implies $\numat{q}\leq
      2^i-1$, and

    \item if $q$ is an ELQs with $\numat{q} \leq s$ that fits the
      positive examples $E^+$ in $E$, then
      there is an $i \leq 2\log(s)$ with $q \in M_i$.

  \end{enumerate}
  Points~(i) and (ii) imply that $\Amf_1$ returns a fitting ELQ that is
  only polynomially larger than a smallest fitting ELQ 
  and is thus Occam. Indeed, let a
  smallest fitting ELQ $q^*$ be of size~$||q^*||=s$ 
  and let $q$ be the ELQ returned by $\Amf_1$.
  By~(ii), $\Amf_1$ discovers $q^*$ after
  $2\log(s)$ rounds, which by definition of $\Amf_1$ implies that $q$ is
  returned after at most $2\log(s)$ rounds. It thus follows from~(i)
  that the returned ELQ $q$ satisfies
  \[\numat{q}\leq 2^{2\log(s)}-1\in O(s^2).\]
  Consequently, there is a polynomial $p$ such that
  $\Hmc^{\Amf_1}(\Omc,\Sigma,s,m)$ contains only ELQs $q$ with
  $\numat{q}\leq p(s)$. There are at most $|\Sigma|^{p(s)}$ such
  queries and since $2^{|S|}$ queries are needed to shatter a set $S$,
  the VC-dimension of $\Hmc^{\Amf_1}(\Omc,\Sigma,s,m)$ is at most
  $\log (|\Sigma|^{p(s)})=p(s)\cdot \log(|\Sigma|)$. It then follows
  from Lemma~\ref{lem:ourblumer} that $\Amf_1$ is a sample-efficient PAC
  learning algorithm.
%

  It thus remains to prove Points~(i) and~(ii).  Point~(i) can be
  shown by induction on $i$. For the induction start, it suffices to
  recall that $M_1 = \{ \top \}$ and $\numat{\top}=1$. For $i>1$,
  $M_i$ consists of ELQs $p$ with $\numat{p}\leq 2\numat{q}+1$ for some ELQ
  $q \in M_{i-1}$. Applying the induction hypothesis, we obtain
  \begin{align*}
    \numat{p} &\leq 2 \cdot \left( 2^{i-1}-1\right)+1=2^i-1. 
    %
    %
   %
  \end{align*}

  For Point~(ii), let $q$ be an ELQ with $\numat{q}\leq s$ fitting $E^+$.
  It suffices to show that there is a $\rho_1,\Sigma$-refinement sequence
  $q_1,\ldots,q_k$ from $\top$ to $q$ with $k\leq 2\log (s)$.

  Let $p_1,\ldots,p_m$ be a $\Sigma$-specialization sequence from
  $\top$ to~$q$, that is, $p_{i+1}$ is a downward neighbor of
  $p_i$ (equivalently: $p_{i+1}\in \rho_2(p_i,\Sigma)$), for all $i$.
  Inductively define $q_1,q_2\ldots$ as follows: 
  \begin{itemize}

    \item $q_1 = p_1 = \top$, and 

    \item for even numbers $j\geq 2$, let $\ell$ be maximal with
      $\numat{p_\ell}\leq 2\numat{q_{j-1}}+1$, and set:
      \begin{itemize}

	\item $q_{j}=p_\ell$, and

	\item $q_{j+1}=p_{\ell+1}$ if $\ell<m$.

      \end{itemize}
    Stop if $q_j$ or $q_{j+1}$ is $p_m$.

  \end{itemize}
  By construction, $\numat{q_{j}}\leq 2\numat{q_{j-1}}+1$ for even $j\geq 2$. Moreover,
  for odd $j\geq 2$, $q_{j}\in \rho_2(q_{j-1},\Sigma)$ and thus
  Lemma~\ref{lem:smallneighbor} implies that $\numat{q_{j}}\leq
  2\numat{q_{j-1}}+1$. Finally, observe that the construction ensures that
  $\numat{q_{j+2}}>2\numat{q_{j}}$, for all odd $j$ and thus the process stops
  after at most $2\log(s)$ steps.
\end{proof}


\section{Details for Section~\ref{sect:elpl}}

From $E_\Omc$ and the bound $n$, SPELL constructs a propositional $\varphi =
\varphi_1 \land \varphi_2$ that is satisfiable if and only if there is an ELQ
$q$ over $\Sigma = \mn{sig}(E_\Omc)$ with $n - 1$ existential restrictions that fits $E_\Omc$.

Intuitively, $\varphi_1$ makes sure that every model of $\varphi$ encodes an ELQ
$q$ in the variables $c_{i, A}, x_{i, r}, y_{i, j}$ 
as described in the main part. Recall that we use an arrangement of
$n$ concepts $C_1,\ldots,C_n$. In what follows, we let $q$ denote the
encoded ELQ and assume that $\Dmc_q$ has individuals $1,\ldots,n$
(as indicated by $C_1,\ldots,C_n$) with
$1$ being the root. For encoding a proper arrangement, $\varphi_1$,
it contains the following clauses for each $i$ with $2 \leq i \leq n$:
\begin{align}
  & \bigvee^{i - 1}_{j = 1} y_{j, i} \\
  &\neg y_{j_1, i} \lor \neg y_{j_2, i} &&\text{for all $j_1, j_2$
  with $1 \leq j_1 < j_2 < i$} \\
  & \bigvee_{r \in \Sigma \cap \NR} x_{i, r} \\ 
  & \neg x_{i, r} \lor \neg x_{i, r'} && \text{for all $r, r' \in \Sigma \cap \NR$ with $r \neq r'$}
\end{align}
Clauses~(1) and~(2) ensure that $C_j$ appears in exactly one
$C_i$, $i<j$ as a conjunct of the form $\exists r.C_j$ for some role
name~$r$. Clauses~(3) and~(4) ensure that there is a unique such role name. 
%
 
The formula $\varphi_2$ makes sure that $q$ fits $E_\Omc$ by enforcing
that 
\begin{itemize}

  \item[$(\ast)$] the variables $s_{i, a}$ are true in a model of
    $\varphi$ iff $a\in C_i(\Dmc)$ iff $(\Dmc_q,i)\preceq (\Dmc,a)$,

\end{itemize}
where $\Dmc$ is the disjoint union of all databases that occur in
$E_\Omc$. To achieve this, we implement the properties of simulations
in terms of clauses. The challenge is to capture both directions of
the ``iff'' in~$(\ast)$ in an efficient way. 

For all $a \in \mn{adom}(\Dmc)$, $\mn{type}(a)$ is the set
$\{ A \in \NC \mid A(a) \in \Dmc\}$. Let $\mn{TP}=\{\mn{type}(a)\mid
a\in\mn{adom}(\Dmc)\}$ be the set of all types in that occur \Dmc. We
introduce auxiliary variables $t_{i,\tau}$, for every $1\leq i\leq n$
and $\tau\in \mn{TP}$ with the intuition that $t_{i,\tau}$ is true iff
all concept names that occur as a conjunct in $C_i$ are contained in $\tau$. This is enforced
by including in $\varphi_2$ the following clauses for all $i$ with $1
\leq i \leq n$ and all types $\tau\in \mn{TP}$:
\begin{align}
  & \neg t_{i, \tau} \lor \neg c_{i, A} && \text{for all $A \in
  (\Sigma \cap \NC \setminus \tau)$} \\
  & t_{i, \tau} \lor \bigvee_{A \in (\Sigma \cap \NC \setminus \tau)}
  c_{i, A}
\end{align}
The simulation condition on concept names is now enforced by the following
clauses, for $i$ with $1\leq i\leq n$ and all $a \in \mn{adom}(\Dmc)$:
\begin{align}
  & \neg s_{i, a} \lor t_{i, \mn{type}(a)}
\end{align}
This captures, however, only the ``only if''-direction of the ``iff'' in~$(\ast)$.
To implement the other direction and the simulation condition for role
names, we introduce further auxiliary variables $d_{i,j,a}$ ($d$ as in
\emph{defect} to indicate non-simulation) with the intuitive meaning
that $d_{i,j,a}$ is true iff $(\Dmc_q,i)\not\preceq (\Dmc,a)$ and there
is an $r$-successor to $j$ that is not simulated in any $r$-successor
of $a$ (the $r$ is uniquely determined by $j$ by Clauses~(3)
and~(4)). This is achieved by the
following clauses for all $i,j$ with $1\leq i<j\leq n$, 
 and $r\in\Sigma\cap \NR$, $a\in\mn{adom}(\Dmc)$, and all $r(a,b)\in
\Dmc$:
\begin{align}
  & s_{i, a} \lor \neg t_{i, \mn{type}(a)} \lor \bigvee^n_{k = i + 1}
  d_{i, k, a} \\
  & d_{i, j, a} \lor \neg y_{i, j} \lor \neg x_{j, r} \lor \bigvee_{r(a, c) \in \Dmc}
  s_{j, c} \\
   & \neg s_{i, a} \lor \neg d_{i, j, a} \\
  %
  %
  & \neg d_{i, j, a} \lor y_{i, j} \\
  & {\neg d_{i, j, a} \lor \neg x_{j, r} \lor \neg s_{j, b}}  
\end{align}
As an example, Clause~(11) can be read as follows: if there is a
defect $d_{i,j,a}$, then $y_{i,j}$ must be true, meaning that 
$\exists r.C_j$ occurs as a conjunct in $C_i$.

It can be verified that the number of variables is $O(n^2+n\cdot|\Dmc|)$,
the number of clauses is
$O(n^3\cdot|\Sigma|\cdot |\mn{adom}(\Dmc)|)$ and that the overall size of the formula is
$O(n^3\cdot |\Sigma|\cdot|\Dmc|)$ as well.

%
%
%

\medskip

Next, we give details on the additional clauses that break some 
symmetries in $\varphi$.
As an example for these symmetries, consider the ELQ $\exists r.\exists s.\top \sqcap \exists t.\top$.
In our encoding, it may be represented by the concepts
\[
  C_1 = \exists r.C_2 \sqcap \exists t. C_3,\ C_2 = \exists s.C_4,\ C_3 = C_4 = \top
\]
or equivalently by the concepts
\[
  C'_1 = \exists r.C'_2 \sqcap \exists t. C'_4,\ C'_2 = \exists s.C'_3,\ C'_3 = C'_4 = \top.
\]
These different representations correspond to different models of $\varphi_1$.
Consider the underlying graphs $G_{\Dmc_C} =
(\mn{adom}(\Dmc_C), \{(a, b) \mid r(a, b) \in \Dmc_C\})$ of concepts, where 
$\Dmc_{C}$ is the concept $C$ viewed as a pointed database.
Note that $G_{\Dmc_{C_1}} = G_{\Dmc_{C_1'}}$.
The only difference between the arrangements $C_1,\ldots,C_4$ and
$C_1',\ldots,C_4'$ comes from assigning them in a different way to the
vertices of
$G_{\Dmc_{C_1}}$.

To avoid this, we add in Round $n$ of bounded fitting
clauses that permit for every tree-shaped graph $G$
with $n$ vertices only a single canonical assignment of the
concepts $C_1, \ldots C_{n}$ to the vertices of
$G$. It suffices to consider tree-shaped graphs since $G_C$ is
tree-shaped for every \EL-concept $C$.
To produce the clauses, we enumerate (outside the SAT solver) all
possible tree-shaped graphs with
$n$ vertices. For each such graph $G$, we introduce a propositional
variable $x_G$ and encode (in a straightforward way) that $x_G$ is
true iff $C_1,\ldots,C_n$ are assigned to the vertices of $G$ in the
canonical way. We then assert (with a big disjunction) that one of the
$x_G$ has to be satisfied. 
However, note that there are exponentially many
possible graphs and therefore we only add these clauses if $n < 12$, to
avoid spending too much time and undoing the benefit of breaking this
symmetry. It is an interesting research question how to break even
more symmetries. 

\section{Size-restricted fitting for \texorpdfstring{\EL}{EL} and \texorpdfstring{\ELI}{ELI}}

We analyze the complexity of the size-restricted fitting problem for
ELQs, for ELIQs, and for the OMQ language $(\ELI,\text{ELIQ})$. Recall
that universal databases in the sense defined before
Proposition~\ref{prop:tbox} do not exist for the latter, and in fact
not even for 
$(\EL,\text{ELIQ})$. We discuss
this a bit further at the end of this section.  Recall that we
generally assume unary coding of the input $k$ to the size-restricted
fitting problem.\footnote{This seems more relevant in the context of
  the current paper: it suffices for the size-restricted fitting
  problem to be in \NPclass with $k$ coded in unary to enable a SAT
  approach to bounded fitting with the SAT formulas being of size
  polynomial in (the size of the data examples and) $k$.} An
investigation of the problem under binary coding is left as future
work; a good starting point for this are results
in~\cite{DBLP:conf/lics/JungPWZ19,DBLP:conf/aaai/JungLW20}.

\begin{lemma} \label{lem:sizerestricted-EL}
  The following problems are \NPclass-complete:
  \begin{itemize}

    \item the size-restricted fitting problem for ELQs;

    \item the problem of deciding given a set of labeled examples $E$
      and a number $k$, whether there is an ELQ that fits $E$ and that
      uses at most $k$ existential restrictions.

  \end{itemize}
\end{lemma}

\begin{proof} 
  The arguments are essentially identical, so we give the proof only for the
  size-restricted fitting problem.

  For the \NPclass upper bound, let $E,k$ be an input to
  the size-restricted fitting problem.  Observe that we can guess in
  polynomial time an ELQ $q$ with $||q||\leq k$ and verify in
  polynomial time that $a\in q(\Dmc)$ for all $(\Dmc,a,+)\in E$ and
  $a\notin q(\Dmc)$ for all $(\Dmc,a,-)\in E$. The latter is true
  since query evaluation of ELQs is possible in \PTime. 

  For \NPclass-hardness, recall that the fitting problem for every
  class of unary conjunctive queries that includes all ELQs is
  \NPclass-hard~\cite{DBLP:journals/corr/abs-2208-10255}.
  The proof of that statement is by reduction from
  3CNF-satisfiability. In more detail, a given 3CNF-formula $\varphi$
  with $m$ variables is translated to a collection of labeled data
  examples $E$ such that $\varphi$ is satisfiable iff $E$ has a
  fitting ELQ of size $p(m)$ for some fixed polynomial $p$. Thus, it
  actually constitutes a reduction to the size-restricted fitting
  problem for ELQs.
\end{proof}

\begin{restatable}{theorem}{thmsizerestrictedeli}
  \label{thm:size-restricted-eli}
  The size-restricted fitting problem is
  \NPclass-complete  
  for $\text{ELIQs}$
  and \ExpTime-complete for $(\ELI,\text{ELIQ})$.
\end{restatable}

\begin{proof}
  We start with the case without ontologies. For the \NPclass upper
  bound, let $E,k$ be an input to the size-restricted fitting problem.
  Observe that we can guess in polynomial time an ELIQ
  $q$ with $||q||\leq k$ and verify in polynomial time that $a\in
  q(\Dmc)$ for all $(\Dmc,a,+)\in E$ and $a\notin q(\Dmc)$ for all
  $(\Dmc,a,-)\in E$. The latter is true since query evaluation of ELIQs
  is possible in \PTime. 

  For \NPclass-hardness, recall that the fitting problem for every
  class of unary conjunctive queries that includes all ELQs is
  \NPclass-hard~\cite{DBLP:journals/corr/abs-2208-10255}.
  The proof of that statement is by reduction from
  3CNF-satisfiability. In more detail, a 3CNF-formula $\varphi$
  with $m$ variables is translated to a collection of labeled data
  examples $E$ such that $\varphi$ is satisfiable iff $E$ has a
  fitting ELIQ of size $p(m)$ for some fixed polynomial $p$. Thus, it
  actually constitutes a reduction to the size-restricted fitting
  problem.

  \medskip We now consider the OMQ language $(\ELI,\text{ELIQ})$. We
  show \ExpTime-hardness by reduction from subsumption
  w.r.t.\ \ELI-ontologies which
  is known to be~\ExpTime-hard~\cite{DBLP:conf/owled/BaaderLB08}. 
  Let $\Omc,A,B$ be an input to the subsumption problem, that is, the
  question is to decide whether \mbox{$\Omc\models A\sqsubseteq B$.}
  Construct a copy $\Omc'$ of $\Omc$ by replacing every concept name
  $X\in\text{sig}(\Omc)\setminus\{B\}$ with a fresh concept name
  $X'$ and every role name $r\in\text{sig}(\Omc)$ with a fresh role name $r'$. Clearly, $\Omc\models A\sqsubseteq B$ iff $\Omc'\models
  A'\sqsubseteq B$. Then let $E$ consist of the two labeled examples
  \[( \{ A(a)\},a,+) \quad\quad \text{ and }\quad\quad (\{A'(a)\},a,+).\]
  Then $\Omc\models A\sqsubseteq B$ iff there is a fitting of
  $E$ w.r.t.\ $\Omc\cup\Omc'$ iff $B$ is a fitting of $E$ w.r.t.\ $\Omc\cup\Omc'$.

  For the \ExpTime-upper bound let $\Omc,E,k$ be an input to
  the size-restricted fitting problem. We provide a Turing reduction
  to subsumption w.r.t.\ \ELI-ontologies. In the reduction, we
  enumerate all ELIQs $q$ with $||q||\leq k$, and test for each
  whether it fits $E$ w.r.t.\ \Omc using an oracle for answering
  instance queries over \ELI knowledge bases.
  Since there are only exponentially many candidates $q$, each test
  whether $q$ fits $E$ w.r.t.\ \Omc uses only $|E|$ calls to the
  oracle. Since $k$ is encoded in unary, the inputs to the oracle are
  of polynomial size. Finally, as the oracle itself runs in
  exponential time~\cite{DBLP:conf/owled/BaaderLB08}, the \ExpTime-upper bound follows. 
\end{proof}
Let us return to the issue of universal databases in
$(\ELI,\text{ELIQ})$. As mentioned above, universal databases as
defined in the main body of the paper do not exist for this OMQ
language. For bounded fitting, however, one might consider a weaker
notion. For every database~\Dmc, \ELI-ontology~\Omc, and $k \geq 1$,
one can compute a database~$\Umc_{\Dmc,\Omc,k}$ that is
\emph{$k$-universal for ELIQs} in the sense that
$a\in q(\Dmc\cup \Omc)$ iff $a\in q(\Umc_{\Dmc,\Omc,k})$ for all ELIQs
$q$ with at most $k$ existential restrictions (or of size at most
$||k||$, a stronger condition) and all $a\in\mn{adom}(\Dmc)$. We do
not give a detailed construction here, but only mention that such a
database can be obtained from an infinite tree-shaped universal
database by cutting off at depth $k$.  What this means is that while
we do not have available a universal database that works for
\emph{all} rounds of bounded fitting, for each single round $k$ we can
compute a $k$-universal database to be used in that round.  In
contrast to the case of $(\EL,\text{ELQ})$, these $k$-universal
databases may become exponential in size. One may hope, though, that
their size is still manageable in practical cases. Note that keeping
the ontology and treating it in the SAT encoding is not an option due to
the \ExpTime-hardness identified by
Theorem~\ref{thm:size-restricted-eli}.

\end{document}